\documentclass[11pt]{article}

\usepackage{subfig}
\usepackage{tikz}
\usepackage{dsfont}
\usepackage{pgf,pgfplots}
\usepackage{mathrsfs}
\def\colorful{1}

\usepackage{amsthm,amsfonts,amsmath,amssymb,epsfig,color,float,graphicx,verbatim, enumitem}
\usepackage{multirow}
\usepackage{algorithm}
\usepackage[noend]{algpseudocode}

\newif\ifhyper\IfFileExists{hyperref.sty}{\hypertrue}{\hyperfalse}
\hypertrue
\ifhyper\usepackage{hyperref}\fi

\usepackage{enumitem}

\usepackage{environ}

\newcommand{\repeattheorem}[1]{\begingroup
  \renewcommand{\thetheorem}{\ref{#1}}\expandafter\expandafter\expandafter\theorem
  \csname reptheorem@#1\endcsname
  \endtheorem
  \endgroup
}

\newcommand{\repeatlemma}[1]{\begingroup
  \renewcommand{\thelemma}{\ref{#1}}\expandafter\expandafter\expandafter\lemma
  \csname replemma@#1\endcsname
  \endlemma
  \endgroup
}

\NewEnviron{replemma}[1]{\global\expandafter\xdef\csname replemma@#1\endcsname{\unexpanded\expandafter{\BODY}}\expandafter\lemma\BODY\unskip\label{#1}\endlemma
}

\NewEnviron{reptheorem}[1]{\global\expandafter\xdef\csname reptheorem@#1\endcsname{\unexpanded\expandafter{\BODY}}\expandafter\theorem\BODY\unskip\label{#1}\endtheorem
}

\def\nnewcolor{1}
\ifnum\nnewcolor=1
\newcommand{\nnew}[1]{{#1}}
\fi
\ifnum\nnewcolor=0
\newcommand{\nnew}[1]{#1}
\fi

\ifnum\colorful=1
\newcommand{\new}[1]{{ #1}}
\newcommand{\newA}[1]{{ #1}}

\else
\newcommand{\new}[1]{{#1}}

\fi

\newtheorem{theorem}{Theorem}[section]

\newtheorem{lemma}[theorem]{Lemma}
\newtheorem{informal theorem}[theorem]{Theorem (informal statement)}

\newtheorem{corollary}[theorem]{Corollary}
\newtheorem{claim}[theorem]{Claim}
\newtheorem{fact}[theorem]{Fact}

\theoremstyle{definition}
\newtheorem{definition}[theorem]{Definition}
\newcommand{\eqdef}{\stackrel{{\mathrm {\footnotesize def}}}{=}}

\newtheorem{innercustomgeneric}{\customgenericname}
\providecommand{\customgenericname}{}
\newcommand{\newcustomtheorem}[2]{\newenvironment{#1}[1]
  {\renewcommand\customgenericname{#2}\renewcommand\theinnercustomgeneric{##1}\innercustomgeneric
  }
  {\endinnercustomgeneric}
}

\newcustomtheorem{customlem}{Lemma}
\newcustomtheorem{customclm}{Claim}
\newcustomtheorem{customfc}{Fact}

\newcommand{\lp}{\left}
\newcommand{\rp}{\right}

\newcommand\snorm[2]{\left\| #2 \right\|_{#1}}
\renewcommand\vec[1]{\mathbf{#1}}
\DeclareMathOperator*{\Prob}{\mathbf{Pr}}
\DeclareMathOperator*{\E}{\mathbf{E}}

\def\d{\mathrm{d}}

\newcommand{\sample}[2]{#1^{(#2)}}
\newcommand{\tr}{\mathrm{tr}}
\newcommand{\bx}{\mathbf{x}}
\newcommand{\by}{\mathbf{y}}

\newcommand{\bw}{\mathbf{w}}

\newcommand{\Sp}{\mathbb{S}}

\newcommand{\err}{\mathrm{err}}

\newcommand{\R}{\mathbb{R}}

\newcommand{\Z}{\mathbb{Z}}

\newcommand{\eps}{\epsilon}

\newcommand{\pr}{\mathbf{Pr}}
\newcommand{\poly}{\mathrm{poly}}
\newcommand{\var}{\mathbf{Var}}

\newcommand{\sgn}{\mathrm{sign}}
\newcommand{\sign}{\mathrm{sign}}

\newcommand{\calL}{{\cal L}}

\newcommand{\opt}{\mathrm{OPT}}
\newcommand{\D}{\mathcal{D}}

\newcommand{\Ind}{\mathds{1}}
\newcommand{\1}{\Ind}

\newcommand{\littlesum}{\mathop{\textstyle \sum}}

\newcommand{\wt}{\widetilde}
\newcommand{\wh}{\widehat}
\usetikzlibrary{angles, quotes}
\usepackage{pgfplots}

\newcommand{\tsya}{{A}}
\newcommand{\tsyb}{{\alpha}}
\newcommand{\cona}{{B}}
\newcommand{\conb}{{\beta}}

\newcommand{\dotp}[2]{\left\langle #1, #2 \right\rangle}

\newcommand{\wstar}{\bw^{\ast}}
\newcommand{\citep}{\cite}
\newcommand{\x}{\vec x}

\newcommand{\bounded}{(L, R, \cona, \conb)}
\newcommand{\boundedU}{(L, R, U, \cona, \conb)}
\title{Learning Halfspaces with Tsybakov Noise}

\author{
Ilias Diakonikolas\thanks{Supported by NSF Award CCF-1652862 (CAREER), a Sloan Research Fellowship, and a DARPA  Learning with Less Labels (LwLL) grant.}\\
University of Wisconsin-Madison\\
{\tt ilias@cs.wisc.edu}\\
\and
Vasilis Kontonis\\
University of Wisconsin-Madison\\
{\tt kontonis@wisc.edu }\\
\and
Christos Tzamos\\ University of Wisconsin-Madison\\
{\tt tzamos@wisc.edu}
\and
Nikos Zarifis\thanks{Supported in part by a DARPA  Learning with Less Labels (LwLL) grant.}\\
University of Wisconsin-Madison\\
{\tt zarifis@wisc.edu}\\
}

\begin{document}

\maketitle

\begin{abstract}
We study the efficient PAC learnability of halfspaces in the presence of Tsybakov noise.
In the Tsybakov noise model, each label is independently flipped with some probability
which is controlled by an adversary. This noise model significantly generalizes the Massart noise model,
by allowing the flipping probabilities to be arbitrarily close to $1/2$ for a fraction of the samples.

Our main result is the first non-trivial PAC learning algorithm for this problem under
a broad family of structured distributions --- satisfying certain concentration
and (anti-)anti-concentration properties --- including log-concave distributions.
Specifically, we given an algorithm that achieves misclassification error $\eps$
with respect to the true halfspace, with quasi-polynomial runtime dependence in $1/\eps$.
The only previous upper bound for this problem --- even for the special case of log-concave distributions ---
was doubly exponential in $1/\eps$ (and follows via the naive reduction to agnostic learning).

Our approach relies on a novel computationally efficient procedure to certify whether a candidate solution is near-optimal,
based on semi-definite programming. We use this certificate procedure as a black-box and turn it
into an efficient learning algorithm by searching over the space of halfspaces
via online convex optimization.
\end{abstract}

\setcounter{page}{0}
\thispagestyle{empty}
\newpage

\section{Introduction} \label{sec:intro}
\subsection{Background and Motivation} \label{ssec:background}
Halfspaces (or Linear Threshold Functions) are one of the most fundamental
concept classes in machine learning and have been an object of intense investigation
since the beginning of the field~\cite{Rosenblatt:58, Novikoff:62, MinskyPapert:68}.
The study of their efficient learnability in various models,
starting with the Perceptron algorithm in the 1950s~\cite{Rosenblatt:58},
has played a central role in the development of machine learning, and has led to important tools such
as SVMs~\cite{Vapnik:98} and Adaboost~\cite{FreundSchapire:97}.

Formally, an (origin-centered) halfspace is any function $f: \R^d \to \{ \pm 1\}$
of the form $f(\bx) = \sgn(\langle \bw, \bx \rangle)$, where the vector $\bw \in \R^d$
is called the weight vector of $f$.
(The function $\sgn: \R \to \{\pm 1\}$
is defined as $\sgn(t) = 1$ if $t \geq 0$ and $\sgn(t) = -1$ otherwise.)
While the sample complexity of learning halfspaces is
understood in a range of models, the computational complexity of the problem
depends critically on the choice of model.
In the noise-free setting, halfspaces are known to be efficiently learnable in the distribution-independent
PAC model~\cite{val84} via linear programming (see, e.g.,~\cite{MT:94}).
On the other hand, the picture is much less clear in the presence of noisy data.
Despite significant theoretical progress over the past two decades,
several fundamental algorithmic questions in the noisy setting are still a mystery.

In this work, we study the algorithmic problem of learning halfspaces under the
Tsybakov noise condition~\cite{tsybakov2004optimal}, a challenging noise model
that has been extensively studied in the statistics and machine learning communities.
While the information-theoretic aspects of learning with Tsybakov noise have been largely
characterized, prior to this work, the computational aspects of this broad problem
had remained wide open.

We now proceed to define this noise model.
The Tsybakov noise condition prescribes that the label of each example
is independently flipped with some probability which is controlled by an adversary.
Importantly, this noise condition allows the flipping probabilities to
be {\em arbitrarily close to $1/2$} for a fraction of the examples. More formally, we have the following definition:

\begin{definition}[PAC Learning with Tsybakov Noise] \label{def:tsybakov-learning}
Let $\mathcal{C}$ be a concept class of Boolean-valued functions over $X= \R^d$,
$\mathcal{F}$ be a family of distributions on $X$, $0< \eps <1$ be the error parameter,
and $0 \leq \tsyb < 1$, $\tsya> 0$ be parameters of the noise model.

Let $f$ be an unknown target function in $\mathcal{C}$.
A {\em Tsybakov example oracle}, $\mathrm{EX}^{\mathrm{Tsyb}}(f, \mathcal{F})$, works as follows:
Each time $\mathrm{EX}^{\mathrm{Tsyb}}(f, \mathcal{F})$ is invoked, it returns a
labeled example $(\bx, y)$, such that:
(a) $\bx \sim \D_{\bx}$, where $\D_{\bx}$ is a fixed distribution in $\mathcal{F}$, and
(b) $y = f(\bx)$ with probability $1-\eta(\bx)$ and $y = -f(\bx)$ with probability $\eta(\bx)$.
Here $\eta(\bx)$ is an {\em unknown} function  that satisfies the Tsybakov noise condition
with parameters $(\tsyb, \tsya)$. That is, for any $0<t \leq 1/2$, $\eta(\bx)$ satisfies
the condition $\pr_{\bx \sim \D_{\bx}}[\eta(\bx) \geq 1/2 - t] \leq \tsya \, t^{\frac{\tsyb}{1-\tsyb}}$.

Let $\D$ denote the joint distribution on $(\bx, y)$ generated by the above oracle.
A learning algorithm is given i.i.d. samples from $\D$ and its goal is to output
a hypothesis function $h: X \to \{\pm 1\}$ such that with high probability $h$ is $\eps$-close to $f$,
i.e., it holds $\pr_{\bx \sim \D_{\bx}} [h(\bx) \neq f(\bx)] \leq \eps$.
\end{definition}

The noise model of Definition~\ref{def:tsybakov-learning} was first proposed in~\cite{MT99}
and subsequently refined in~\cite{tsybakov2004optimal}.
Since these initial works, a long line of research in statistics and learning theory has focused on understanding
a range of statistical aspects of the model in various settings
(see, e.g.,~\cite{tsybakov2004optimal, BBL05, BJM06, BalcanBZ07, Hanneke2011, HannekeY15}
and references therein). Ignoring computational considerations,
it is known that the class of halfspaces is learnable in this model with $\poly(d, 1/\eps^{1/\alpha})$ samples,
where $d$ is the dimension and $\eps$ is the error to the target halfspace.

On the other hand, the algorithmic question has remained poorly understood.
Roughly speaking, the only known algorithms in this noise model (for any non-trivial concept class in high dimension)
are the ones that follow via the naive reduction to agnostic learning.
We also note that efficient algorithms for learning halfspaces were previously known
in more structured random noise models, including random classification noise
and bounded (Massart) noise. (See Section~\ref{ssec:related} for a detailed summary of prior work.)

\subsection{Our Contributions} \label{ssec:results}

As explained in the above discussion (also see Section~\ref{ssec:related}),
obtaining computationally efficient learning algorithms in the presence of
Tsybakov noise in {\em any} non-trivial setting --- that is, for any natural concept class and
under any distributional assumptions --- has been a long-standing open problem in learning theory.
In this work, we make the first progress on this problem. Specifically,
we give a learning algorithm for halfspaces that succeeds under a class of
well-behaved distributions (including log-concave distributions) and runs in time
{\em quasi-polynomial} in $1/\eps$.

We start by describing the distribution family for which our algorithm succeeds.

\begin{definition}[Bounded Distributions] \label{def:bounds}
For any set of parameters $L, R, B, \beta >0$,  an isotropic (i.e., zero mean and identity covariance)
distribution $\D_{\bx}$ on $\R^d$ is called {\em $\bounded$-bounded} if for any
projection $(\D_{\bx})_V$ of $\D_{\bx}$ on a $2$-dimensional subspace $V$,
the corresponding pdf $\gamma_V$ on $\R^2$ satisfies the following properties:
\begin{enumerate}
\item We have that $\gamma_V(\bx) \geq L$, for all $\bx \in V$ such that $\snorm{2}{\bx} \leq R$ (anti-anti-concentration).
  \item For any $t>0$, we have that $\pr_{\bx \sim \gamma_V} [\snorm{2}{\bx} \geq t] \leq \cona \exp(-\conb t)$ (concentration).
\end{enumerate}
Moreover, if there exists $U>0$ such that for all $\bx \in V$ we have that $\gamma_V(\bx) \leq U$ (anti-concentration), then
the distribution $\D_\bx$ is called $\boundedU$-bounded.
\end{definition}

Definition~\ref{def:bounds} specifies the concentration and (anti-)anti-concentration properties on the underlying
data distribution that are needed to prove the correctness of our algorithm. We note that the sample complexity
and runtime of our algorithm depends on the values of these parameters.

For concreteness, we state a simplified version of our main result for the case that $L, R, U, B, \beta$
are positive universal constants. We call such distributions {\em well-behaved}.
We note that the class of well-behaved distributions is quite broad. In particular, it is easy to
show (Fact~\ref{fact:logcon}) that every isotropic log-concave distribution is well-behaved.
Moreover, the concentration and anti-concentration conditions of Definition~\ref{def:bounds}
do not require a specific nonparametric constraint for the underlying density function,
and are satisfied by many reasonable continuous distributions.

We show:

\begin{theorem}[Learning Halfpaces with Tsybakov Noise] \label{thm:main-informal}
Let $\mathcal{C}$ be the class of origin-centered halfspaces and $\mathcal{F}$
be a family of well-behaved distributions on $\R^d$. There is an algorithm with the following
behavior: On input the error parameter $\eps>0$ and oracle access to a Tsybakov
example oracle $\mathrm{EX}^{\mathrm{Tsyb}}(f, \mathcal{F})$ with parameters $(\tsyb,\tsya)$,
where $f \in \mathcal{C}$ is the target concept, the algorithm draws
$N= d^{O\left((1/\tsyb^2)\log^2(1/\eps)\right)}$ labeled examples,
runs in $\poly(N, d)$ time, and computes a hypothesis $h \in \mathcal{C}$
that with high probability is $\eps$-close to $f$.
\end{theorem}

See Theorem~\ref{thm:main_pac} for a more detailed statement that takes into account the dependence
on the parameters $L, R, U, B, \beta$.

Some comments are in order.
Theorem~\ref{thm:main-informal} provides the first algorithm for learning halfspaces (or any other concept class)
in the presence of Tsybakov noise with running time beating that of agnostically learning the class.
For the special case of log-concave distributions, the best sample complexity and running time bounds
that can be obtained via agnostic learning are $d^{2^{\poly\left(1/\eps^{1/\alpha}\right)}}$.
(See Section~\ref{ssec:related} for a detailed summary.)
That is, we provide a nearly doubly exponential improvement on the $\eps$-dependence, even for fixed $\alpha>0$.
Moreover, since our algorithm does not require log-concavity,
it applies to distribution families for which no sub-exponential in $d$ upper bound was previously known.
\nnew{Interestingly, recent work~\cite{DKZ20, GGK20} has given Statistical Query (SQ) lower bounds of $d^{\poly(1/\eps)}$ for agnostically learning halfspaces, even under Gaussian marginals. Since our algorithm runs in $d^{\mathrm{polylog}(1/\eps)}$ time, this implies a computational separation between agnostic learning and Tsybakov learning for the class of halfspaces.}

Finally, we note that the exponential dependence on $1/\alpha$ is to some extent unavoidable,
since $\Omega(d/\eps^{1/\alpha})$ samples are information-theoretically necessary to solve our problem.

The main question left open by our work is whether the quasi-polynomial dependence on $1/\eps$
can be improved to polynomial, i.e., whether a $\poly(d, 1/\eps^{1/\alpha})$ time algorithm exists.
We leave this as an outstanding open problem.

\subsection{Overview of Techniques} \label{ssec:techniques}

In this subsection, we give an intuitive description of our techniques
that lead to Theorem~\ref{thm:main-informal} in tandem with a brief
comparison to prior techniques and why the fail in our context.

It is instructive to begin by explaining where algorithms for the related problem of learning with Massart noise fall apart.
The Massart noise model corresponds to the special case of Tsybakov noise where the label of each example $\bx$
is independently flipped with probability $\eta(\bx) \leq \eta$, where $\eta<1/2$ is a parameter of the model.
A line of work has developed efficient algorithms for learning halfspaces in this model, with the recent
works~\cite{zhang2020efficient, DKTZ20} being the state-of-the-art. (See Section~\ref{ssec:related} for more details.)

We start by briefly describing the underlying idea behind several previous algorithms for
learning halfspaces with Massart noise~\cite{zhang2020efficient, DKTZ20}. 
These algorithms are typically iterative: In each iteration $t$, we have a current guess $\vec{w}$ 
for the normal vector $\wstar$ to the true halfspace,
and our goal is to perform a local step to improve our guess (in expectation). 
To perform these updates, the algorithms aim to boost the contribution of the disagreement region $A$ 
between the halfspaces corresponding to $\vec{w}$ and $\wstar$. This is achieved by considering 
points only around a small band around $\vec{w}$, i.e., all $\x$ with $|\dotp{\vec{w}}{\x}| < T$. 
This idea suffices to obtain efficient algorithms for the Massart noise model under well-behaved (e.g., log-concave) 
distributions as the total contribution of those points is amplified.

For the case of Tsybakov noise however, the situation is much more challenging.
Even though the probability mass of the points in region $A$ increases 
by restricting to a band around the current guess, it does not guarantee that the angle between $\vec{w}$ and $\wstar$ improves.
This is because in the Tsybakov noise model,
it is possible that all points in region $A$ have flipping probabilities $\eta(\bx) \approx 1/2$,
which grow closer to $1/2$ the more the band shrinks. Thus, even though the conditional probability of region $A$ 
increases with smaller band size $T$, the signal that these points provide to improve 
the angle may not be strong enough to overcome the effect that the remaining points have.

Our main idea to overcome this obstacle is to increase the contribution of points in region $A$ by
appropriately reweighting them (see Figure~\ref{fig:Tsybakov_Regions}). A key observation that drives our algorithm \new{(see Fact~\ref{obs:optimal_condition})}
is to find a weighting scheme that {\em certifies} whether a given guess $\vec w$ is (near-)optimal.
In more detail, if there exists a \new{non-negative} weighting function $F(\x)$
such that $\E_{(\x,y) \sim \D}[F(\x) y\; \sign(\dotp{\vec w}{\x})] < 0$,
then the weight vector $\vec w$ is not optimal. \new{Conversely,}  if $\vec w$ is not optimal,
a weighting function $F$ that makes the above expectation negative always exists
(take for example the indicator of the disagreement region between $\vec w$ and $\vec w^{\ast}$).

Our first technical contribution is making the aforementioned certificate algorithmic.
In more detail, we show that in order to certify that a guess $\vec w$ is $\epsilon$-far from optimal,
it suffices to consider weighting functions of a particular form,
equal to the square of a multivariate polynomial restricted on a band close to $\vec w$.
In particular, we show \new{(Theorem~\ref{lem:polynomial_certificate})}
that it suffices to consider polynomials of degree at most $k = O(\log^2(1/\epsilon)/\alpha^2)$.
We provide an explicit construction of such a multivariate polynomial with bounded coefficients,
making critical use of Chebyshev polynomials.

Given this structural result, we can efficiently check the validity of a particular guess
by searching all functions of the aforementioned form. Drawing sufficiently many samples
so that all functions in the class converge uniformly, we can identify a good weighting (if one exists)
by solving a semidefinite program to check the required condition over all squares of polynomials of degree-$k$.
The sample complexity required to find our certificate
is $d^{O(k)}$ and can be achieved in sample-polynomial time \new{(Lemma~\ref{lem:sample_sdp})}.

We note that while our algorithm searches over multivariate polynomials that certify the error of our estimate,
our approach differs significantly from other approaches for learning halfspaces
by approximating them by polynomial threshold functions, like the $L_1$-regression algorithm of \cite{KKMS:08}.
Our use of polynomials is done in order to certify whether a candidate halfspace is sufficiently accurate,
instead of searching a larger class of hypotheses. Remaining within the class of halfspaces allows
us to use geometric properties of the underlying data distributions and the setting we consider,
like the relationship of the misclassification error and the angle between the guess and the optimal halfspace.
Additionally, while the $L_1$-regression can be written as a linear program, our approach requires
searching over squares of polynomials and inherently relies on solving SDPs for obtaining a certificate.

Finally, turning the above algorithm for obtaining certificates into a learning algorithm is not immediate.
To achieve this, we rely on online convex optimization with a similar approach to the one used in \cite{zhang2020efficient}.
In contrast to an offline method like stochastic gradient descent, online convex optimization
allows us to change the distribution of examples with which we penalize the guess,
and the distribution is allowed to depend on the current guess.
For every guess $\vec w$,
we compute a loss function according to the reweighted distribution of points given by our certificate.
We set up the objective so that any guess that is not close to optimal incurs a large loss,
while the optimal guess always incurs a very small loss. By the guarantees of online convex optimization,
after few iterations, the average loss of our guesses must be very close to the optimal loss. This
means that one of the guesses must be near-optimal \new{(see Lemma~\ref{lem:expectation_error})}.
This property will cause the certificate algorithm to accept this guess as close to optimal.
A complication that arises in designing the loss function
is that guessing $0$ must give a large loss compared to the optimal, which we ensure
by making the loss sufficiently negative at the optimal linear classifier.

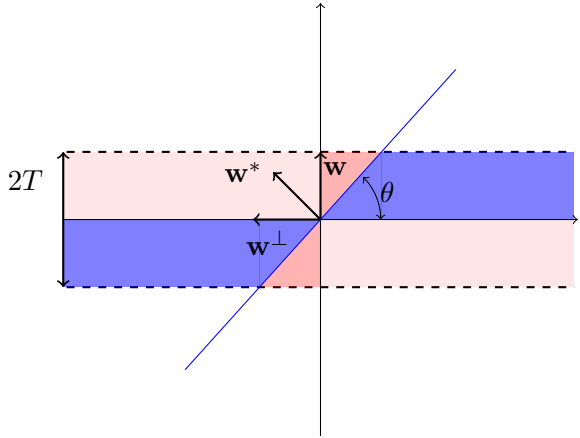
\begin{figure}
\iffalse
	  \begin{minipage}[t]{0.47\textwidth}

	\centering
	\begin{tikzpicture}[scale=0.9]
\coordinate (start) at (0.5,0);
	\coordinate (center) at (0,0);
	\coordinate (end) at (0.5,0.5);

\draw[black,dashed, thick](-3.75,1) -- (3.75,1);
	\draw[black,dashed, thick](-3.75,-1) -- (3.75,-1);
\draw[fill=blue, opacity=0.3,draw=none] (0,0) -- (0.9 ,1)--(0.9,0);
		\draw[fill=blue, opacity=0.3,draw=none] (0.9,0) rectangle (3.75,1);
			\draw[fill=blue, opacity=0.3,draw=none] (0,0) -- (-0.9 ,-1)--(-0.9,0);
		\draw[fill=blue, opacity=0.3,draw=none] (-0.9,0) rectangle (-3.8,-1);
				\draw[fill=blue, opacity=0.3,draw=none] (0,0) rectangle (-3.8,1);
						\draw[fill=blue, opacity=0.3,draw=none] (0,0) rectangle (3.75,-1);
						\draw[fill=red, opacity=0.3,draw=none] (0,0) -- (0.9 ,1)--(0,1);
						\draw[fill=red, opacity=0.3,draw=none] (0,0) -- (-0.9 ,-1)--(0,-1);
\draw[<->,thick] (-3.8,-1) -- (-3.8,1) node[black,left=5mm,below=1mm] {$2T$};
\draw[->] (-3.8,0) -- (3.8,0) node[anchor=north west,black] {};
	\draw[->] (0,-3.2) -- (0,3.2) node[anchor=south east] {};
	\draw[thick,->] (0,0) -- (-0.7,0.7) node[anchor= south east,below,left=0.1mm] {$\wstar$};
	\draw[blue] (-2,-2.22) -- (2,2.22);
	\draw[thick ,->] (0,0) -- (0,1) node[right=2mm,below] {$\bw$};
		\draw[thick ,->] (0,0) -- (-1,0) node[right=2mm,below] {$\bw^\bot$};
	\pic [draw, <->,
	angle radius=8mm, angle eccentricity=1.2,
	"$\theta$"] {angle = start--center--end};
\end{tikzpicture}
	\caption{The region $A$ (``blue'') and the region $B$ (``red'') under Massart Noise}
	\label{fig:Massart_Regions}
	\end{minipage}
\fi
	\centering
	\begin{minipage}[t]{\textwidth}

	\centering
		\begin{tikzpicture}[scale=0.9]
\coordinate (start) at (0.5,0);
	\coordinate (center) at (0,0);
	\coordinate (end) at (0.5,0.5);

\draw[black,dashed, thick](-3.75,1) -- (3.75,1);
	\draw[black,dashed, thick](-3.75,-1) -- (3.75,-1);
\draw[fill=blue, opacity=0.5,draw=none] (0,0) -- (0.9 ,1)--(0.9,0);
	\draw[fill=blue, opacity=0.5,draw=none] (0.9,0) rectangle (3.75,1);
	\draw[fill=blue, opacity=0.5,draw=none] (0,0) -- (-0.9 ,-1)--(-0.9,0);
	\draw[fill=blue, opacity=0.5,draw=none] (-0.9,0) rectangle (-3.8,-1);
	\draw[fill=red, opacity=0.1,draw=none] (0,0) rectangle (-3.8,1);
	\draw[fill=red, opacity=0.1,draw=none] (0,0) rectangle (3.75,-1);
	\draw[fill=red, opacity=0.3,draw=none] (0,0) -- (0.9 ,1)--(0,1);
	\draw[fill=red, opacity=0.3,draw=none] (0,0) -- (-0.9 ,-1)--(0,-1);
\draw[<->,thick] (-3.8,-1) -- (-3.8,1) node[black,left=5mm,below=1mm] {$2T$};
\draw[->] (-3.8,0) -- (3.8,0) node[anchor=north west,black] {};
	\draw[->] (0,-3.2) -- (0,3.2) node[anchor=south east] {};
	\draw[thick,->] (0,0) -- (-0.7,0.7) node[anchor= south east,below,left=0.1mm] {$\wstar$};
	\draw[blue] (-2,-2.22) -- (2,2.22);
	\draw[thick ,->] (0,0) -- (0,1) node[right=2mm,below] {$\bw$};
	\draw[thick ,->] (0,0) -- (-1,0) node[right=2mm,below] {$\bw^\bot$};
	\pic [draw, <->,
	angle radius=8mm, angle eccentricity=1.2,
	"$\theta$"] {angle = start--center--end};
\end{tikzpicture}
	\caption{The disagreement region $A$ (``blue'') of the halfspaces $\bw$ and $\wstar$.
	\new{Our reweighting
	boosts points in region $A$: lower opacity means lower weight.
      }
      }
	\label{fig:Tsybakov_Regions}
	\end{minipage}

\end{figure}

\subsection{Related Work}\label{ssec:related}

It is instructive to compare the Tsybakov noise model with two other classical noise
models, namely the agnostic model~\cite{Haussler:92, KSS:94} and the
bounded (or Massart) noise model~\cite{Sloan88, Massart2006}.
The Tsybakov noise model lies in between these two models.

In the agnostic model~\cite{Haussler:92, KSS:94}, the learner is given access to iid labeled
examples from an arbitrary distribution $\D$ on labeled examples
$(\bx, y) \in \R^d \times \{\pm 1\}$ and the goal of the learner is to output a hypothesis $h$
such that the misclassification error $\err_{0-1}^{\D}(h) \eqdef \pr_{(\bx, y) \sim \D}[h(\bx) \neq y]$
is as small as possible. In more detail, we want to achieve $\err_{0-1}^{\D}(h) \leq \opt+\eps$,
where $\opt \eqdef \inf_{g \in \mathcal{C}} \err_{0-1}^{\D}(g)$ is the minimum possible misclassification
error by any function in the class $\mathcal{C}$. Agnostic noise is the most challenging noise
model in the literature. Without assumptions on the marginal distribution $\D_{\bx}$ on the (unlabaled)
points, (even weak) agnostic learning is known to be computationally intractable~\cite{GR:06, FGK+:06short, Daniely16}.

On the other hand, if $\D_{\bx}$ is known to be well-behaved, in a precise sense, dimension-efficient agnostic algorithms
are known. Specifically, the $L_1$-regression algorithm of~\cite{KKMS:08} agnostically learns halfspaces
under the standard Gaussian and, more generally, any isotropic log-concave distribution,
with sample complexity and runtime $d^{m(1/\eps)}$, for an appropriate function $m$.
In more detail, if $\D_{\bx}$ is the standard Gaussian $N(0, I)$, then $m(1/\eps) = \tilde{\Theta}(1/\eps^2)$ 
(see, e.g.,~\cite{DGJ+:10, DKNfocs10})
and if $\D_{\bx}$ is any isotropic log-concave distribution, then $m(1/\eps) = 2^{\Theta(\poly(1/\eps))}$.
These runtime bounds are tight for the $L_1$-regression approach,
as they rely on the minimum degree of certain polynomial approximations of the univariate sign function.
\nnew{Moreover, recent work~\cite{DKZ20, GGK20} has shown Statistical Query lower bounds 
of $d^{\poly(1/\eps)}$ for agnostically learning halfspaces, even under Gaussian marginals.}

Prior to this work, the only known algorithms for Tsybakov noise are the ones obtained via the straightforward reduction
to agnostic learning. Specifically, by applying the $L_1$-regression algorithm~\cite{KKMS:08} for $\eps' = \Theta(\eps^{1/\alpha})$
in place of $\eps$, where $\alpha \in (0, 1]$ is the Tsybakov noise parameter of Definition~\ref{def:tsybakov-learning},
we have \new{(see, e.g., Corollary~\ref{lem:tsybakov_01_dis})} that the output hypothesis $h$ satisfies
$\Pr_{\bx \sim \D_{\bx}}[h(\bx) \neq f(\bx)] \leq \eps$. This straightforward reduction
leads to algorithms with runtimes $d^{\poly\left(1/\eps^{1/\alpha}\right)}$
for Gaussian marginals, and $d^{2^{\poly\left(1/\eps^{1/\alpha}\right)}}$ for log-concave marginals.

We acknowledge a related line of work~\cite{KLS09, ABL17, Daniely15, DKS18a} that gave efficient algorithms for
learning halfspaces with agnostic noise under similar distributional assumptions. While these algorithms run in time
$\poly(d/\eps)$, they achieve a ``semi-agnostic'' error guarantee of $O(\opt)+\eps$ --- instead of $1 \cdot \opt +\eps$.
This guarantee is significantly weaker for our purposes and cannot be used to obtain a hypothesis
that is arbitrarily close to the target halfspace.

The bounded (Massart) noise model~\cite{Sloan88, Massart2006} is the special case of Tsybakov noise,
where  an adversary can flip the label of each example $\bx$
independently with probability $\eta(\bx) \leq \eta$, for some parameter $\eta<1/2$.
This noise model has attracted significant attention in recent years.
A long line of work, initiated by~\cite{AwasthiBHU15}, has obtained computationally efficient algorithms
for PAC learning halfspaces with Massart noise to arbitrary accuracy
(under distributional assumptions)~\cite{AwasthiBHZ16, ZhangLC17, YanZ17, MangoubiV19, zhang2020efficient, DKTZ20}.
Recent works developed polynomial-time algorithms (in all relevant parameters) under
log-concave~\cite{zhang2020efficient, DKTZ20},
$s$-concave, and other structured distributions~\cite{DKTZ20}.
\new{These algorithms inherently fail for the more challenging
Tsybakov noise model, and new ideas are needed for this more general setting.}

\nnew{We note that the recent work~\cite{DGT19} developed the first
computationally efficient weak learner for halfspaces with Massart noise in the
distribution-independent setting. The approach of~\cite{DGT19} can be adapted
to give a weak learner for halfspaces under Tsybakov noise as well, but cannot
directly lead to an arbitrarily close approximation to the true halfspace.}

Finally, it should be noted that this work is part of the broader agenda of designing robust estimators
for a range of generative models with respect to various noise models. 
A recent line of work~\cite{KLS09, ABL17, DKKLMS16, LaiRV16, DKK+17, DKKLMS18-soda, 
DKS18a, KlivansKM18, DKS19, DKK+19-sever}
has given efficient robust estimators for a range of learning tasks (both supervised and unsupervised)
in the presence of a small constant fraction of adversarial corruptions.

 \newcommand{\capfun}{\mathrm{cap}}
\newcommand{\CLR}{\mathrm{CappedLeakyRelu}}

\section{Preliminaries}

For $n \in \Z_+$, let $[n] \eqdef \{1, \ldots, n\}$.  We will use small
boldface characters for vectors.  For $\bx \in \R^d$ and $i \in [d]$, $\bx_i$
denotes the $i$-th coordinate of $\bx$, and $\|\bx\|_2 \eqdef
(\littlesum_{i=1}^d \bx_i^2)^{1/2}$ denotes the $\ell_2$-norm of $\bx$.
We will use $\langle \bx, \by \rangle$ for the inner product of $\bx, \by \in
\R^d$ and $ \theta(\bx, \by)$ for the angle between $\bx, \by$. We will also denote $\1_A$ to be the characteristic function of the set $A$, i.e., $\1_A(\x)= 1$ if $\x\in A$ and $\1_A(\x)= 0$ if $\x\notin A$.

Let $\vec e_i$ be the $i$-th standard basis vector in $\R^d$.
For $d\in \mathbb{N}$, let $\Sp^{d-1} \eqdef \{\bx \in \R^d:\|\bx\|_2 = 1 \}$ and ${\cal V} \eqdef \{\bx \in \R^d:\|\bx\|_2 \leq 1 \}$.
Let $\Pi_U(\vec x)$ be the projection of $\vec x$ onto subspace
$U \subset \R^d$. For a subspace $U\subset\R^d$, let $U^{\perp}$ be the orthogonal complement of $U$.

Let $\E[X]$ denote the expectation of random variable $X$ and
$\pr[\mathcal{E}]$ the probability of event $\mathcal{E}$.

We consider the binary classification setting where labeled examples $(\bx,y)$ are drawn
i.i.d. from a distribution $\D$ on $\R^d \times \{ \pm 1\}$.
We denote by $\D_{\bx}$ the marginal of $\D$ on $\vec x$.
The misclassification error of a hypothesis $h: \R^d \to \{\pm 1\}$ (with respect to $\D$) is
$\err_{0-1}^{\D}(h) \eqdef \pr_{(\bx, y) \sim \D}[h(\bx) \neq y]$. The zero-one error between
two functions $f, h$ (with respect to $\D_{\bx}$) is
$\err_{0-1}^{\D_{\bx}}(f, h) \eqdef \pr_{\bx \sim \D_{\bx}}[f(\bx) \neq h(\bx)]$.

For a square matrix $\vec M$, we say that $\vec M$ is positive semi-definite if only if all the eigenvalues of $\vec M$ are non-negative. For $m\in \Z_+$, we denote $\mathcal{S}^m$ the set of symmetric matrices of dimension $m$. For an $m$-dimensional square matrix $\vec A$, let $\tr(\vec A)$ be its trace.

Let $S=(s_1,s_2,\ldots,s_d)$ be a $d$-dimensional multi-index vector, where for all $i\in[d]$, $s_i$ is non-negative integer. We denote $|S|=\sum_{i=1}^d s_i$ and for a $d$-dimensional vector $\vec w=(\vec w_1,\vec w_2,\ldots,\vec w_d)$, we denote $\vec w^S = \prod_{i=1}^d \vec w_i^{s_i}$.

For a degree-$k$ multivariate polynomial $p(\vec x)=\sum_{S:|S|\leq k} C_{S}\vec \x^S$, let $\snorm{2}{p}\eqdef \sqrt{ \sum_{S:|S|\leq k} C_{S}^2}$ and $\snorm{1}{p}\eqdef \sum_{S:|S|\leq k} |C_{S}|$ .

 \section{Certifying Optimality}\label{sec:section3}

In this section, we describe an efficient way to test whether a given
candidate hypothesis $\vec w$ is close to the optimal hypothesis $\vec
w^*$. Our approach is based on the following observation.

\begin{fact}\label{obs:optimal_condition}
  For any $F: \R^d \mapsto \R_+$ and any distribution $\D$ on $\R^d \times \{\pm 1\}$ that satisfies the Tsybakov noise condition, it holds that
  \begin{equation}
  \label{eq:certificate_positivity}
\E_{(\vec x, y) \sim \D}[ F(\vec x) \dotp{\wstar}{\vec x} y]
\geq 0\;.
\end{equation}
\end{fact}
\begin{proof}
	We have that
\begin{align*}
	\E_{(\vec x, y) \sim \D}[ F(\vec x) \dotp{\wstar}{\vec x} y]&=	\E_{\vec x \sim \D_{\bx}}
	[ F(\vec x) |\dotp{\wstar}{\bx}| (1-\eta(\x))] - \E_{\vec x \sim \D_{\bx}}
	[ F(\vec x) |\dotp{\wstar}{\bx}|	 \eta(\x)]
\\&= 	\E_{\vec x \sim \D_{\bx}}
[F(\vec x) |\dotp{\wstar}{\bx}|
(1-2 \eta(\vec x))]	\geq 0\,,
\end{align*}
where we used the fact that $\eta(\bx) \leq 1/2$ and $F(\x)\geq 0$.
\end{proof}
From Fact~\ref{obs:optimal_condition}, we see that, given a hypothesis vector
$\bw$ that is not optimal, there exists a non-negative function that will make
the expression of Equation~\eqref{eq:certificate_positivity} negative. One such
function is $F(\bx)=\1\{\sign(\dotp{\bw}{\x})\neq \sign(\dotp{\wstar}{\x})\}$,
in which case we have $\E_{(\vec x, y) \sim \D}[ F(\vec x) \dotp{\bw}{\vec x}
y]= - \E_{\vec x \sim \D_{\bx}} [|\dotp{\bw}{\bx}| (1-2 \eta(\vec x))]<0$.
Since we cannot efficiently search over the space of all non-negative functions, we need to
restrict our search space of certifying functions to some parametric class,
ideally with a small number of parameters. In
Section~\ref{sub:certificate_existence}, we show that considering
\nnew{squares of low-degree polynomials} suffices. In
Section~\ref{sub:certificate_optimization}, we show that we can efficiently
search in the space of \nnew{(squares of) low-degree polynomials} and 
find one that will make the expression of Equation~\eqref{eq:certificate_positivity} negative.

\subsection{Existence of a Low-Degree Polynomial Certificate}
\label{sub:certificate_existence}
We start by showing that given a candidate hypothesis $\vec w$ that is ``far"
from being optimal, that is the angle $\theta(\vec w, \vec w^*)$ is bounded
away from zero, we can construct a \emph{low complexity} certificate $F$ that
will satisfy $\E_{(\bx, y) \sim \D}[F(\bx) \dotp{\bw}{\bx} y] < 0$. In
particular, we construct a certificate that is the product of a square of a low degree
non-negative polynomial and an indicator function that depends on the hypothesis
$\bw$. This result is formally stated in the lemma bellow, which is the main
result of this subsection.

\begin{theorem}[Low Complexity Certificate] \label{lem:polynomial_certificate}
Let $\D$ be a distribution on $\R^d \times \{\pm 1\}$ that satisfies the Tsybakov noise condition with parameters $(\tsyb,\tsya)$ 
and the marginal $\D_{\bx}$ on $\R^d$ is $\bounded$-bounded. Fix any $\theta \in (0, \pi/2]$.
Let $\wstar \in \Sp^{d-1}$ be the normal vector to the optimal halfspace and
$\wh{\bw} \in \Sp^{d-1}$ be such that $\theta(\wh{\bw}, \wstar) \geq \theta$.
There exists polynomial $p: \R^d \mapsto \R$ of degree
$$
k = O\left( \frac{1}{\tsyb^2 R \conb} \log^2\left(\frac{\cona \tsya}{ L R \theta} \right) \right)
$$
satisfying $\snorm{2}{p}^2 \leq d^{O(k)}$ such that
\begin{align*}
  \E_{(\vec x,y) \sim \D}
  \left[
  p(\vec x)^2
  ~
  \1{\{0\leq \dotp{\vec w}{\vec x}\leq \theta R/4\}}
  ~
  y
  \dotp{\vec w}{\vec x}
\right]
  \leq - \frac{\theta R}{4} \;.
\end{align*}
\end{theorem}

We are going to use the following simple fact about Tsybakov noise that shows
that large probability regions will also have large integral even if we
weight the integral with the noise function $1-2\eta(\bx)>0$.  Notice that
larger noise $\eta(\bx)$ makes $1-2\eta(\bx)$ closer to $0$, and therefore tends to
reduce the probability mass of the regions where $\eta(\bx)$ is large.
A similar lemma can be found in \cite{tsybakov2004optimal}. Since the definition of $\eta(\bx)$
is slightly different than ours, we provide the proof for completeness in Appendix~\ref{ap:claim33}.
\begin{lemma} \label{lem:tsybakov_expectation}
Let $\D$ be a distribution on $\R^{d} \times \{\pm 1\}$ that satisfies the Tsybakov noise condition with parameters $(\tsyb,\tsya)$.
Then for every measurable set $S \subseteq \R^d$ it holds
$
\E_{\vec x \sim D_\bx}[ \1_S(\bx) (1- 2 \eta(\bx))]
  \geq
  C_{\tsyb}^\tsya
  \lp( \E_{\vec x \sim \D_\bx}[ \1_S(\bx)] \rp)^{\frac 1 \tsyb}
$, where
$C_{\tsyb}^\tsya = \tsyb \left( \frac{1-\tsyb}{\tsya} \right)^{\frac{1-\tsyb}{\tsyb}}$.
\end{lemma}
Using the lemma above, we can bound from below and above the $\err_{0-1}^\D(h)$ 
with the $\err_{0-1}^{\D_\x}(h,f)$ between our current hypothesis $h$ and the optimal $f$.

\begin{corollary}\label{lem:tsybakov_01_dis}
Let $\D$ be a distribution on $\R^{d} \times \{\pm 1\}$ that satisfies the Tsybakov noise condition with parameters $(\tsyb,\tsya)$ 
and $f(\x)$ be the optimal halfspace. Then for any halfspace $h(\x)$, it holds
\begin{align*}
\pr_{(\x,y)\sim \D}[h(\x)\neq y]
&\leq \pr_{(\x,y)\sim \D}[f(\x)\neq y] +\pr_{\x\sim \D_\x}[h(\x)\neq f(\x)]\\ 
&\qquad\qquad\mathrm{ and } \\ 
\pr_{(\x,y)\sim \D}[h(\x)\neq y]  &\geq \pr_{(\x,y)\sim \D}[f(\x)\neq y] + C_{\tsyb}^\tsya \pr_{\x\sim\D_\x}[h(\x)\neq f(\x)]^{\frac{1}{\tsyb}}\;.
\end{align*}
\end{corollary}
\begin{proof}
	Let $S=\{\x\in\R^d: f(\x)\neq h(\x)\}$ then
	\begin{align*}
		\pr_{(\x,y)\sim \D}[h(\x)\neq y] &=\E_{(\bx, y) \sim \D}[\1\{h(\x)\neq y\}] =\E_{\bx \sim \D_\x}[\1\{h(\x)\neq f(\x)\}(1-\eta(\x))] +\E_{\bx \sim \D_\x}[\1\{h(\x)= f(\x)\}\eta(\x)]\\ &=\E_{\bx \sim \D_\x}[\1\{h(\x)\neq f(\x)\}(1-2\eta(\x))] +\E_{\bx \sim \D_\x}[\eta(\x)]\;.
	\end{align*}  The first inequality follows from the fact that $1-2\eta(\x)\leq 1$ and the second one from Lemma~\ref{lem:tsybakov_expectation}.
\end{proof}
Central role in our construction play the Chebyshev polynomials.
In the next fact, we collect the properties of Chebyshev polynomials
that we are going to use in our argument, and we prove some of them in Appendix~\ref{ap:factclaim}.
\begin{fact}[Chebyshev Polynomials \cite{HM02}]
  \label{fct:chebyshev}
  We denote by $T_k(t)$ the  degree-$k$ Chebyshev polynomial of the first kind.  It holds
	\begin{align*}
	T_k(t)
	=
	\begin{cases}
	  \cos(k\arccos t)\;, & |t| \le 1 \\
          \frac12 \bigg( \Big(t-\sqrt{t^2-1} \Big)^k + \Big(t+\sqrt{t^2-1} \Big)^k \bigg)\;, \qquad & |t| \ge 1\;. \\
        \end{cases}
	\end{align*}
	Moreover, it holds
	$\snorm{2}{T_k}^2 \leq 2^{6k+\log k +4}$.
\end{fact}

Given a univariate polynomial $p(t)$, the following simple
lemma bounds the blow-up of the square norm of the multivariate
polynomial $q(\bx) = p(\dotp{\vec w}{\vec x})$.  We also give a simple bound on the
coefficient norm blow-up under shift of the argument of a univariate polynomial.

\def\FunctionF(#1){(-1 + 18 *(1 + 2/27 *(-3 + #1))^2 - 48 *(1 + 2/27 *(-3 + #1))^4 +
	32 *(1 + 2/27 *(-3 + #1))^6)^2}

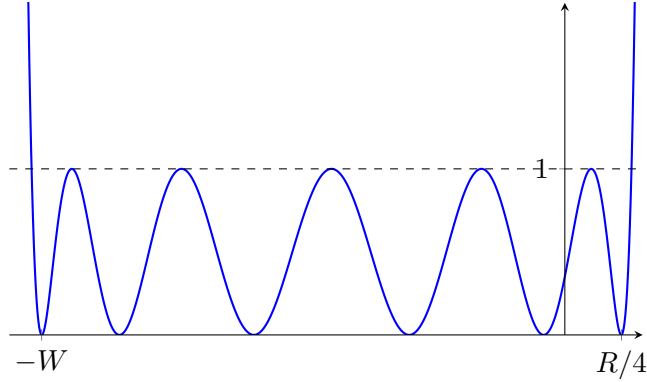
\begin{figure}
\centering
\begin{tikzpicture}[scale=1]
\begin{axis}[
axis y line=center,
axis x line=middle,
axis on top=true,
xmin=-25,
xmax=3.5,
ymin=0,
ymax=2,
height=6.0cm,
width=10.0cm,
xticklabels={$-W$,$R/4$},xtick={-23.55,2.55},
ytick={0,1},
]
\addplot [domain=-25:3.5, samples=700, mark=none,  thick, blue] {\FunctionF(x)};
\addplot [domain=-25:3.5, samples=100, mark=none,dashed, black] {1};
\end{axis}
\end{tikzpicture}
\caption{Plot of the polynomial $(T_k(g(t)))^2$ used in the proof of
  Theorem~\ref{lem:polynomial_certificate}.  Observe that this polynomial
  boosts the contribution of points in the blue region of
  Figure~\ref{fig:Tsybakov_Regions}: points in $A_2$ have significantly
  boosted contribution because their density is lower bounded by some
  constant and the polynomial takes very large values in $A_2$, see
  Fact~\ref{fact:regions}.  In $A_0$, even though the polynomial
  has large value, the exponential tails of the distribution cancel
  the contribution of these points (given that $W$ is sufficiently large).
}
\end{figure}
\begin{lemma}
  \label{lem:polynomial_norms}
  Let $p(t) = \sum_{i=0}^k c_i t^i$ be a  degree-$k$ univariate polynomial.
  Given $\vec w \in \R^d$ with $\snorm{2}{\vec w} \leq 1$, define the multivariate polynomial
  $q(\vec x) = p(\dotp{\vec w}{\vec x}) = \sum_{S: |S| \leq k} C_S \bx^S$. Then we have that
  $
  \sum_{S:|S| \leq k} C_S^2 \leq d^{2k} \sum_{i=0}^k c_i^2 \,.
  $
  Moreover, let $r(t) = p(a t + b) = \sum_{i=0}^k d_i t^i$ for some $a, b \in \R$.  Then
  $
  \snorm{2}{r}^2
  \leq
  (2 \max(1,a) \max(1,b))^{2k} \snorm{2}{p}^2\, .
  $
\end{lemma}
The proof of this lemma is given in Appendix~\ref{ap:factclaim}. 
We can now proceed to the proof of the main technical theorem.
\begin{proof}[Proof of Theorem~\ref{lem:polynomial_certificate}]
	Let $V$ be the $2$-dimensional subspace spanned by $\wstar$ and
	$\vec w$. To simplify notation, let $\theta$ be the angle between $\wstar$ and $\vec
	w$. First, we assume that $\theta\leq \pi/2$. Without loss of generality, assume $\vec w = \vec e_2$ and
	$\wstar = -a \vec e_1 + b \vec e_2$, where $\vec e_1,\vec e_2$ are the standard basis vectors of $\R^2$.
	For some parameter $W > 0$ to be specified later, we define the
	linear transformation
	$$ g(t) = 1+ 2 \frac{t - R/4}{W + R/4}. $$
	Set $p(\vec x) =  T_k(g(\vec x_1))$, where $T_k$ is the degree-$k$ Chebyshev polynomial of Fact~\ref{fct:chebyshev},
	and define the following partition of $\R^d$
	\begin{align*}
	  A_0 = \{\vec x: \vec x_1 \in [-\infty, -W]\},
	  &~~~~~
	  A_1 = \{\vec x: \vec x_1 \in [-W, R/4] \},~~~ \mathrm{ and }
	  &
	  A_2 &= \{\vec x: \vec x_1 \in [R/4, +\infty] \}\;.
	\end{align*}
	We first investigate the behavior of $p(\vec x)$ in
	each of these three regions.

	\begin{fact}\label{fact:regions}
	  For the polynomial $p(\x)$ defined above, the following properties
	  hold in each region:
		\begin{enumerate}
			\item For all $\x\in A_0$, $p(\x)^2\leq (2g(\x_1))^{2k}$.
			\item For all $\x\in A_1$, $p(\x)^2\leq 1$.
			\item \newA{
			    For all $\x$ such that  $\x_1 \geq R/2$,
			    it holds that
			    $p(\x)^2 \geq
			  \frac12 \left( 1 + \sqrt{\frac{R}{2 W + R/2}}
			  \right)^{2k}
			    $.
			  }
		\end{enumerate}
	\end{fact}
\begin{proof}

By Fact~\ref{fct:chebyshev},
	for the univariate Chebyshev polynomials of degree-$k$, we know that
	for all $t\leq -1$ it holds
	$$
	|T_k(t)| = \left| \frac12 ((t - \sqrt{t^2 -1})^k + (t + \sqrt{t^2 -1})^k) \right|
	\leq  (2 t)^k\,.
	$$
	Observe that for all $\bx \in A_0$, we have $g(\x_1)\leq-1$, thus
	$p(\vec x)^2 \leq (2 g(\vec x_1))^{2 k}$.  For all $\bx \in A_1$, we
	have $ -1\leq g(\x_1)\leq 1$, which leads to $p(\vec x)^2 \leq 1$.
	  \newA{
	Finally, from the definition of the Chebyshev polynomial $T_k$
	(Fact~\ref{fct:chebyshev}), we have that for all $t \geq 0$ it holds
	$$
	T_k(1+t) \geq \frac12 (1+t + \sqrt{t^2 + 2 t})^k
	\geq \frac12 (1 + \sqrt{t})^k.
	$$
	Moreover, all the roots of $T_k(t)$ lie in the interval $[-1,1]$
	and hence, for $t \geq 0$, the polynomial $(T_k(1+t))^2$ is increasing in $t$.
	Therefore, for any $\x$ with $\x_1 \geq R/2$ it holds that
	$$
	p(\x)^2 = T_k(g(\x_1)) \geq T_k(g(R/2)) \geq
	\frac12 \left( 1 + \sqrt{\frac{R}{2 W + R/2}} \right)^{2k}.
	$$ }
\end{proof}
	We bound the expectation
	$ \E_{(\vec x,y) \sim \D}[p(\vec x)^2
	\dotp{\vec w}{\vec x} y\
	  \sign(\dotp{\vec w}{\vec x})
	\1\{0\leq \dotp{\vec w}{\vec x}\leq \frac{\theta R}{4}\}]
	$ in each of the three regions separately.
	We start from $A_0$, where we have
	\begin{align*}
	  I_0 &=
	\E_{(\vec x,y) \sim \D}
	[p(\vec x)^2
	\dotp{\vec w}{\vec x}
	~ y
	~ \1\{\dotp{\vec w}{\vec x} \in [0, \theta R/4] \}
	~ \1_{A_0}(\vec x)]
	\\
	&
	=
	\E_{(\vec x,y) \sim \D}
	[
	T_k(g(\vec x_1))^2
	~ \vec x_2 y
	~ \1\{\vec x_2 \in [0, \theta R/4]\}
	~ \1\{\vec x_1 \leq - W \}
	]
	\\
	&\leq
	\frac{\theta R}{4}
	\E_{(\vec x_1, \vec x_2) \sim \D_V}
	[(2 g(\vec x_1))^{2 k}
	  \1\{\vec x_1 \leq -W \}
	]\;,
	\end{align*}
	where to get the last inequality we used that $ \x_2 \1[ \x_2 \in [0,
	\theta R/4] \leq \theta R/4$ and Item 1 of Fact~\ref{fact:regions}.
	  Using the fact that for any real random variable $X$ it holds
	 $\E[|X|^m] = \int_0^\infty m t^{m-1} \pr[|X| \geq t] \d t$ and the
	 exponential concentration of $\D_V$ (see Definition~\ref{def:bounds}),
	 we obtain
	 \begin{align*}
	\E_{(\vec x_1, \vec x_2) \sim \D_V}
	&[ g(\vec x_1)^{2k} \1\{\vec x_1 \leq -W \}]
	\\
	&=
	\int_0^\infty
	2k t^{2k - 1} \Prob_{(\x_1, \x_2) \sim \D_V}
	[
	|g(\vec x_1) \1\{\vec x_1 \leq -W \}| \leq t
	]
	\d t
	\\
	&=
	   \int_0^1
	   2 k t^{2k -1} e^{- \conb W}
	   \d t
	   \
	   +
	   \
	   \int_1^\infty
	   2 k t^{2k -1}
	   e^{ - \conb \frac{t+1}{2} \left(W + \frac{R}{4}\right) + \conb \frac{R}{4}}
	   \d t\,.
	 \end{align*}
	 We observe that for all $t>1, R>0, W>0$ it holds
	 $$
	 \frac{t+1}{2} \left(W+ \frac{R}{4}\right) - \frac{R}{4}
	 \geq \frac{t W}{2}\,.
	 $$
	 Therefore,
	 \begin{align*}\textit{}
	   \int_1^\infty
	   2 k t^{2k -1}
	   e^{ - \conb \frac{t+1}{2} \left(W + \frac{R}{4}\right) + \conb \frac{R}{4}}
	   \d t
	   &\leq
	   \int_1^\infty
	   2 k t^{2k -1}
	   e^{ - t \conb W/2} \d t
	   \leq \int_0^\infty
	   2 k t^{2k -1}
	   e^{ - t \conb W/2}
	   \d t
	   \leq
	   \left(\frac{W \conb}{2}\right)^{-2k} (2k)!
	   \,.
	 \end{align*}

	 Combining the above inequalities we obtain
	 \begin{align*}
	   I_0
	   &\leq
	   \frac{\theta R \cona 2^{2k}}{4}
	 \left(
	   \int_0^1
	   2 k t^{2k -1} e^{- \conb W}
	   \d t
	   \
	   +
	   \
	   \int_1^\infty
	   2 k t^{2k -1}
	   e^{ - \conb \frac{t+1}{2} \left(W + \frac{R}{4}\right) + \conb \frac{R}{4}}
	   \d t
	 \right)
	   \\
	   &=
	   \frac{\theta R \cona 2^{2k}}{4}
	 \left(
	   e^{- \conb W}
	   +
	   \
	   (W \conb /2)^{-2k} (2k)!
\right)
	   \,.
	 \end{align*}
	 We now set $W = 8 k/ \conb$ and get
	 $$
	 I_0 \leq \frac{\theta R \cona}{4} ( 2^{2k} e^{-8k } +
	  (2k)^{-2 k} (2k)! )
	 \leq \frac{\theta R \cona}{4} ( e^{-6k } + e^{-2 k+1}  \sqrt{2 k}  )
	 \leq \frac{\theta R \cona}{4}\, ,
	 $$
	 where we used Stirling's approximation, i.e., $(2k)! \leq e \sqrt{2k} e^{-2k} (2k)^{2k}$,
	 and the fact that $e^{-6k } + e^{-2 k+1}  \sqrt{2 k}  \leq 1$, for all $k\geq 1$.

	 Bounding the contribution of region $A_1$ is quite simple.
	 Using from Fact~\ref{fact:regions}, that $p(\vec x)^2 \leq 1 $ for all $\x \in A_1$, we obtain
	 $$
	 I_1 =
	 \E_{(\vec x,y) \sim \D}
	[p(\vec x)^2 \dotp{\vec w}{\vec x}
	~ y
	~ \1\{\dotp{\vec w}{\vec x} \in [0, \theta R/4] \}
	~ \1_{A_1}(\vec x)]
	\leq \frac{\theta R}{4}.
	 $$
	 We finally bound the contribution of region $A_2$.
	 We have
	 \begin{align*}
	   I_2 &=
	   \E_{(\vec x,y) \sim \D}
	[p(\vec x)^2 \dotp{\vec w}{\vec x}
	~ y
	~ \1\{\dotp{\vec w}{\vec x} \in [0, \theta R/4] \}
	~ \1_{A_2}(\vec x)]
	\\
	&=
	   - \E_{\vec x \sim \D_\x}
	[p(\vec x)^2 \dotp{\vec w}{\vec x}
	~ (1- 2 \eta(\vec x))
	~ \1\{\dotp{\vec w}{\vec x} \in [0, \theta R/4] \}
	~ \1_{A_2}(\vec x)]
	\\
	&\leq
	   - \E_{\vec x \sim \D_\x}
	[p(\vec x)^2 \dotp{\vec w}{\vec x}
	~ (1- 2 \eta(\vec x))
	~ \1\{\dotp{\vec w}{\vec x} \in [\theta R/8, \theta R/4] \}
	~ \1\{\vec x_1 \geq R/2\} ]
	\\
	&\leq
	   -
	   \frac{\theta R}{8} \newA{T_k(g(R/2))^2}
	   \E_{\vec x \sim \D_\x} [
	~ (1- 2 \eta(\vec x))
	~ \1\{\dotp{\vec w}{\vec x} \in [\theta R/8, \theta R/4] \}
	~ \1\{\vec x_1 \geq R/2\} ]\;,
	 \end{align*}
	 \newA{where we used Item 3 of Fact~\ref{fact:regions}}.
	 Using Lemma~\ref{lem:tsybakov_expectation}, we obtain that
	 $$ \E_{\vec x \sim \D_\x} [
	~ (1- 2 \eta(\vec x))
	~ \1\{\dotp{\vec w}{\vec x} \in [\theta R/8, \theta R/4] \}
	~ \1\{\vec x_1 \geq R/2\} ]
	\geq C_\tsyb^\tsya (L \theta R/\newA{16})^{1/\tsyb}\;.$$
	From \nnew{Item 3} of Fact~\ref{fact:regions}, we obtain
	 $$
	 I_2 \leq
	 - C_\tsyb^\tsya \newA{T_k(g(R/2))^{2}}
	   \frac{\theta R}{8}  \left(L \frac{\theta R^2}{16}\right)^{1/\tsyb}
	   \leq -
	   \frac{\theta R}{4}(\cona+2)
	   \frac{C_\tsyb^\tsya}{2 (\cona+2)}
	   \left(1+ \sqrt{\frac{R}{2 W + R/2}} \right)^{2k}
	   \left(L \frac{\theta R^2}{16}\right)^{\frac{1}{\tsyb}}.
	 $$
	 Using the inequality $1+t \geq e^{t/2}$ for all $t \leq 2$, we obtain
	 that in order to prove that $I_0 + I_1+I_2 \leq - \theta R /4$, it
	 suffices to pick the degree $k$ so that
	 $$
	 \frac{C_\tsyb^\tsya}{2 (\cona+2)}
	 e^{\sqrt{\frac{R k^2}{ 2 W + R/2}}} \left(L \frac{\theta R^2}{16}\right)^{\frac{1}{\tsyb}} \geq
	 1.
	 $$
	By our choice of $W = 8k /\conb$, it follows
	 that setting the degree of the polynomial to
	 $$
	 k = O\left(
	 \frac{1}{\tsyb^2 R \conb} \log^2\left(\frac{\cona \tsya}{ L R \theta}
       \right)
     \right)
	 $$
	 suffices. To complete the proof, we need to provide an upper bound
	 on the magnitude of the coefficients of the polynomial $p$. From
	 Fact~\ref{fct:chebyshev}, we have that $\snorm{2}{T_k(x)}^2 \leq
	 2^{6k+2\log k +4}$.  Using Lemma~\ref{lem:polynomial_norms}, we
	 obtain that $\snorm{2}{T_k(g(x))}^2 \leq 2^{2k} \cdot 2^{6k+2\log k
	 +4} =2^{8k +2\log k +4}$.  Moreover, from the
	 Lemma~\ref{lem:polynomial_norms}, we can derive an upper bound on
	 the square norm of the multivariate polynomial $p$, which is
	 $\snorm{2}{p}^2 \leq d^{2k} 2^{8k +2\log k +4} = d^{O(k)}$.

	  Moreover, for the case where $\pi\geq \theta>\pi/2$, we can prove
	  with the same argument that
	 \begin{align*}
	 \E_{(\vec x,y) \sim \D}
	 \left[
	 p(\vec x)^2
	 ~
	 \1{\{0\leq \dotp{\vec w}{\vec x}\leq \pi R/8\}}
	 ~
	 y
	 \dotp{\vec w}{\vec x}
	 \right]
	 \leq - \frac{\pi R}{8} \;.
	 \end{align*}
This follows from the fact that the expectation over the partitions  $A_0$ and
$A_1$ are at most their values for the case of $\theta=\pi/2$, and \nnew{the expectation over} $A_2$ is the same.

\end{proof}

\subsection{Efficiently Computing the Certificate}
\label{sub:certificate_optimization}
In this section, we show that we can efficiently compute our polynomial
certificate given labeled examples from the target distribution. For the rest of this section, let $Q=d^{\Theta(k)}$ and let $\1_B(\vec x)$ be the indicator function of the region $B=\{\x:0\leq \dotp{\vec w}{\vec x}\leq \theta R/4\}$. Denote by
$\vec m(\bx)$ the vector containing all monomials up to degree
$k$, such that $\vec m_S(\bx)\eqdef\x^S$, indexed by the multi-index $S$ satisfying $|S|\leq k$.
The dimension of $\vec m(\bx) \in \R^m$ is $m = \binom{d+k}{k}$. For a real matrix $\vec A\in\R^{m \times m}$,
we define the following function
\begin{align}
  \label{eq:certificate_objective}
  \mathcal{L}_\bw (\vec{A}) =
  \E_{(\bx, y) \sim \D}\left[ \vec m(\bx)^T \vec A ~
  \vec m(\bx) \1_B(\bx) \dotp{\bw}{\bx} y  \right]
  = \tr\left(\vec A \vec M \right)\;,
\end{align}
where
$
\vec M =  \E_{(\bx, y) \sim \D}\left[ \vec m(\bx)
  \vec m(\bx)^T \1_B(\bx) \dotp{\bw}{\bx} y  \right]$. Notice that $\calL_\bw$ is linear in its variable $\vec A$. 
 From the discussion of the previous subsection, and in particular from
Theorem~\ref{lem:polynomial_certificate}, we know that
if $\theta(\vec w, \wstar) \geq \theta$,
then there exists a polynomial
$p(\vec x)$ and a vector $\vec b$ of coefficients such that $p(\vec x) =
\dotp{\vec b}{\vec m(\vec x)}$ and $\calL_\bw(\vec b \vec b^T) \leq -\theta R/4$.
It follows that there exists a positive semi-definite rank-$1$ matrix $\vec B =
\vec b \vec b^T$ such that $\mathcal{L}_\bw (\vec B) \leq -\theta R/4$.  Moreover,
we have that $\snorm{2}{p^2(\bx)}^2 \leq Q$, which translates to
$\snorm{F}{\vec B}^2 \leq Q$.
Therefore, we can formulate the following semi-definite program,
which is feasible when $\theta(\vec w, \wstar) \geq \theta$.
\begin{align}
  \label{eq:exact_sdp}
  \tr(\vec A \vec M) &\leq-\theta R/4 \nonumber\\
  \snorm{F}{\vec A}^2 &\leq Q \\
  \vec A &\succeq 0 \nonumber
\end{align}
We define $\widetilde{\vec M}
= \frac{1}{N} \sum_{i=1}^N
\vec m(\sample{\bx}{i}) \vec m(\sample{\bx}{i})^T \1_B(\sample{\bx}{i})\sample{y}{i}
\dotp{\bw}{\sample{\bx}{i}}$, the empirical estimate of $\vec M$ using $N$ samples from $\D$. We can now replace the matrix $\vec M$ in
Equation~\eqref{eq:certificate_objective} with the estimate $	\widetilde{\vec M}$ and define the following ``empirical" SDP
\begin{align}
  \label{eq:sample_sdp}
  \tr(\vec A \wt{\vec M}) &\leq -\frac{3 \theta R}{16}\nonumber\\
  \snorm{F}{\vec A}^2 &\leq Q \\
  \vec A &\succeq 0 \nonumber
\end{align}

In the following lemma, we bound the sample size required so that 
$\wt{\vec M}$ is sufficiently close to $\vec M$.

\begin{lemma}[Estimation of $\vec M$]
	\label{lem:empirical_objective_error}
	Let $\Omega = \{\vec A \in \mathcal{S}^m: \vec A\succeq 0,\ \snorm{F}{\vec A}
	\leq Q\}$. There exists an algorithm that draws
	$$N =
	O\left(
	\frac{\cona Q^2}{\eps^2} \frac{(d+k)^{3k+2}}{(\conb/2)^{2k}}
	\log(1/\delta)
	\right)
	$$
	samples from $\D$, runs in $\poly(N,d)$ time and
	with probability at least $1-\delta$ outputs a matrix $\wt{\vec M}$ such
	that
	$$
	\pr
	\left[
	\sup_{\vec A \in \Omega}
	\left| \tr(\vec A \wt{\vec M}) - \tr(\vec A \vec M) \right|
	\geq \eps
	\right]
	\leq 1-\delta\, .
	$$
\end{lemma}
\begin{proof}
	Recall that $\wt{\vec M}$ is the empirical estimate of $\vec M$, that is
	\begin{equation}
	\label{eq:empirical_objective}
	\vec M
	= \E_{(\bx,y)\sim\D}[\vec m(\bx) \vec m(\bx)^T \1_B(\bx)y \dotp{\bw}{\bx}]
	~\text{ and } ~
	\widetilde{\vec M}
	= \frac{1}{N} \sum_{i=1}^N
	\vec m(\sample{\bx}{i}) \vec m(\sample{\bx}{i})^T \1_B(\sample{\bx}{i})\sample{y}{i}
	\dotp{\bw}{\sample{\bx}{i}}\;.
	\end{equation}
	Using the Cauchy-Schwarz inequality, we get
	$$
	\tr\left( \vec A (\vec M - \widetilde{\vec M}) \right)
	\leq
	\snorm{F}{\vec A} \snorm{F}{\vec M - \widetilde{\vec M}}
	\;. $$
	Therefore, it suffices to bound the probability
	that $\snorm{F}{\vec M - \widetilde{\vec M}} \geq \eps/Q$.
	From Markov's inequality, we have
	\begin{equation}
	\label{eq:moment_matrix_markov}
	\pr\left[\snorm{F}{\vec M - \widetilde{\vec M}}
	\geq \eps/Q \right]
	\leq \frac{Q^2}{\eps^2} \E\left[\snorm{F}{\vec M - \widetilde{\vec M}}^2\right]\,.
	\end{equation}
	Using multi-indices $S_1$, $S_2$ that correspond to the monomials
	$\bx^{S_1}, \bx^{S_2}$ (as indices of the matrix $\vec M$), we have
	$$
	\E\left[\snorm{F}{\vec M - \widetilde{\vec M}}^2\right]
	= \sum_{S_1, S_2: |S_1|, |S_2|\leq k}
	(\vec M_{S_1,S_2} - \wt {\vec M}_{S_1,S_2})^2
	= \sum_{S_1, S_2: |S_1|, |S_2|\leq k}
	\var[\wt{\vec M}_{S_1,S_2}]\,.
	$$
	Using the fact that the samples $(\sample{\bx}{i},\sample{y}{i})$ are
	independent, we can bound from above the variance of each entry $(S_1,S_2)$ of
	$\wt{\vec M}$
	\begin{align*}
	\var[\wt{\vec M}_{S_1,S_2}]
	&\leq \frac{1}{N}
	\E_{(\bx, y) \sim \D}
	\left[ \bx^{2(S_1+S_2)}
	\left(\1_B(\bx) \dotp{\bw}{\bx} y\right)^2
	\right] \\
	&\leq
	\frac{1}{N}
	\E_{\bx \sim \D_{\bx}}
	\left[ \bx^{2(S_1+S_2)} \snorm{2}{\bx}^{2}\right] \\
	&\leq
	\frac{1}{N}
	\E_{\bx \sim \D_{\bx}}
	\left[ (\snorm{2}{\bx}^{2})^{|S_1+S_2|+1} \right] \,.
	\end{align*}
	To bound the higher-order moments, we are going to use the (two-dimensional)
	exponential tails of $\D_\bx$ of Definition~\ref{def:bounds}.  For all $t \geq
	t_0$, it holds
	$$
	\pr[\snorm{2}{\bx} \geq t]
	=
	\pr[\snorm{2}{\bx}^2 \geq t^2]
	\leq \sum_{i=1}^d \pr\left[|\x_i|^2 \geq \frac{t^2}{d} \right]
	\leq
	\cona d e^{-\conb t/\sqrt{d}}\, ,
	$$
	where $\conb, \cona$ are the parameters of Definition~\ref{def:bounds}.
	For every $\ell \geq 1$, we have
	\begin{align*}
	\E_{\bx \sim \D_{\bx}}
	\left[ (\snorm{2}{\bx}^{2})^{\ell} \right]
	&= \int_{t=0}^\infty
	2 \ell t^{2 \ell -1} \pr_{\bx \sim \D_\bx}[{\snorm{2}{\bx} \geq t}]
	\d t
	\leq \cona d^{\, \ell+1}{\conb}^{-2 \ell} (2\ell)!\,.
	\end{align*}
	Using the above bound for the variance and summing over all pairs $S_1, S_2$
	with $|S_1|, |S_2| \leq k$, we obtain
	\begin{align}
	\label{eq:momemt_matrix_frobenius_bound}
	\E\left[\snorm{F}{\vec M - \widetilde{\vec M}}^2\right]
	&\leq
	\frac{1}{N}
	\cona d^{\, k+1} {\conb}^{-2 k} (2k)!\ m^2
	=
	\frac{1}{N} \cona d^{\, k+1} {\conb}^{-2 k} (2k)! \binom{d+k}{k}^2
	\nonumber
	\\
	&\leq
	\frac{1}{N}
	\cona {(\conb/2)}^{-2 k} (d+k)^{3k+1} \, ,
	\end{align}
	where we used the inequality $(2n)!/(n!)^2 \leq 4^{n}$. Combining
	Equations~\eqref{eq:moment_matrix_markov} and
	\eqref{eq:momemt_matrix_frobenius_bound} we obtain that with $N \geq \cona Q^2
	{(\conb/2)}^{-2 k} (d+k)^{3k+1}/(4\eps^2)$ samples we can estimate $\vec M$
	within the target accuracy with probability at least $3/4$. To amplify the
	probability to $1-\delta$, we can simply use the above empirical estimate $\ell$
	times to obtain estimates $\sample{\wt {\vec M}}{1}, \ldots, \sample{\wt{\vec
			M}}{\ell}$ and keep the coordinate-wise median as our final estimate. It
	follows that $\ell =O(\log(m/\delta))$ repetitions suffice to guarantee
	confidence probability at least $1-\delta$.

\end{proof}

The following is the main lemma of this subsection,
where we bound the number of samples and the runtime needed
to construct the certificate given samples from the distribution
$\D$.
\begin{lemma}
  \label{lem:sample_sdp}
  Let $\D$ be a distribution on $\R^d \times \{\pm 1\}$ that satisfies the Tsybakov noise condition with parameters $(\tsyb,\tsya)$ and the marginal
  $\D_{\bx}$ on $\R^d$ is $\bounded$-bounded.
  Let $\wstar \in \Sp^{d-1}$ be the normal vector to the optimal halfspace and $\bw \in \Sp^{d-1}$. Fix any $\theta \in (0, \pi/2]$ and assume that
  $\theta(\vec w^*, \vec w) \geq \theta$. Let
  $$
  k = O\left( \frac{1}{\tsyb^2 R \conb} \log^2\left(\frac{\cona \tsya}{ L R \theta} \right) \right),
  $$
  and $Q = d^{\Theta(k)}$.
  There exists an algorithm that draws $N = d^{O(k)} \log(1/\delta)$ samples from $\D$, runs in time $\poly(N,d)$, and
  with probability $1-\delta$
  returns  a positive semi-definite
  matrix $\vec A$ such that
  $\snorm{F}{\vec A}^2 \leq Q$ and
  $\tr(\vec A \vec M) \leq -\theta R/16$.
\end{lemma}
\begin{proof}
  From Lemma~\ref{lem:empirical_objective_error}, we obtain that with $N$
  samples we can get a matrix $\wt{\vec M}$ such that $|\tr(\vec A \wt{\vec
  M} - \tr(\vec A \vec M)| \leq \theta R/16$ with probability at least
  $1-\delta$.
  From Theorem~\ref{lem:polynomial_certificate}, we know
  that with the given bound for $k$ and $\snorm{F}{\vec A}$,
  there exists $\vec A^*$ such that
  $$
  \tr(\vec A^* \vec M) \leq -\theta R/4.
  $$
  Therefore, the SDP~\eqref{eq:exact_sdp} is feasible.  Moreover, from
  Lemma~\ref{lem:empirical_objective_error} we get that
  $$
  \tr(\vec A^* \wt{\vec M}) \leq -\theta R/4 + \theta R/16
  \leq - \frac{3\theta R}{16} \;.
  $$
 Thus, the following SDP is also feasible
\begin{align}
  \tr(\vec A \wt{\vec M}) &\leq -\frac{3\theta R}{16}  \nonumber\\
  \snorm{F}{\vec A}^2 &\leq Q \label{eq:sample_sdp2}\\
  \vec A &\succeq 0 \nonumber
\end{align}
  Since the dimension of the matrix $\vec A$ is smaller than the number of
  samples, we have that the runtime of the SDP is polynomial in the number of
  samples. Solving the SDP using tolerance $\theta R/16$, we obtain
  an almost feasible $\wt{\vec A}$, in the sense that
  $\tr(\wt{\vec A} \wt{\vec M}) \leq - 3 \theta R/16 + \theta R/16
  = - \theta R/8
  $.
  Using again the guarantee of
  Lemma~\ref{lem:empirical_objective_error}, we get that solving the
  SDP~\eqref{eq:sample_sdp2}, we obtain a positive-semi definite
  matrix $\wt{\vec A}$ such that
  $
  \tr(\wt{\vec A} \vec M) \leq -\theta R/8 + \theta R/16
  = -\theta R/16
  $.

\end{proof}
 \section{Learning the Optimal Halfspace via Online Gradient Descent}\label{sec:section4}

In this section, we give a quasi-polynomial time algorithm that
can learn a unit vector $\wh{\vec w}$ with small angle from the normal
vector of the optimal halfspace $\wstar$. Our main result of this
section is the following theorem.

\begin{theorem}[Parameter Estimation under $\bounded$-bounded distributions] \label{thm:main_angle}
Let $\D$ be a distribution on $\R^d \times \{\pm 1\}$ that satisfies the Tsybakov noise condition with parameters $(\tsyb,\tsya)$ and the marginal
$\D_{\bx}$ on $\R^d$ is $\bounded$-bounded.
Moreover, let $\wstar \in \Sp^{d-1}$ be the normal vector to the optimal halfspace.
There exists an algorithm that  draws
$N= d^{O(k)} \log\left(1/\delta\right)   $ examples from $\D$ where $k=O\left(
\frac{1}{\tsyb^2 R \conb}  \log^2\left(\frac{\cona \tsya}{\eps L R}\right)  \right) $, runs in
$\poly(N,d)$ time, and computes a vector $\wh{\vec w}$
such that $\theta(\wh{\vec w}, \wstar) \leq \eps$, with probability $1-\delta$.
\end{theorem}
Note here that we do not need the $U$ bounded assumption for Theorem~\ref{thm:main_angle}. 
This corresponds to an anti-concentration assumption. 
If we have this additional property, we immediately get Theorem~\ref{thm:main_pac}, 
which is the main result of this paper. Specifically, with this additional structure on the distribution, 
one can translate the small angle guarantee of Theorem~\ref{thm:main_angle} to the zero-one loss of the
hypothesis that our algorithm outputs.

\begin{theorem}[PAC-Learning under $\boundedU$-bounded distributions] \label{thm:main_pac}
	Let $\D$ be a distribution on $\R^d \times \{\pm 1\}$ that satisfies the Tsybakov noise condition with parameters $(\tsyb,\tsya)$ and the marginal
	$\D_{\bx}$ on $\R^d$ is $\boundedU$-bounded.
	Moreover, let $\wstar \in \Sp^{d-1}$ be the normal vector to the optimal halfspace.
	There exists an algorithm that  draws
	$N= d^{O(k)} \log\left(1/\delta\right)   $ examples from $\D$ where $k=O\left(
	\frac{1}{\tsyb^2 R \conb}  \log^2\left(\frac{\cona~ U \tsya}{\eps L R \conb}\right)  \right) $, runs in
$\poly(N,d)$ time, and computes a vector $\wh{\vec w}$ such that
 $\err_{0-1}^{\D_{\bx}}(h_{\wh{\bw}}, f) \leq \eps$, with probability $1-\delta$,
where $f$ is the target halfspace.
\end{theorem}
A corollary of the above theorem is that we can PAC learn halfspaces 
when the marginal distribution $\D_\x$ is log-concave. 
The following known fact (see, e.g., Fact A.4 of \cite{DKTZ20}) shows that the family of log-concave 
distributions is indeed $\boundedU$-bounded for constant values of the parameters.
 \begin{fact}\label{fact:logcon}
 An isotropic log-concave distribution on  $\R^d$ is $(2^{-12}, 1/9, e2^{17}, c, 1)$-bounded, where $c$ is an absolute constant.
 \end{fact}
 
 From Thereom~\ref{thm:main_pac} and Fact~\ref{fact:logcon}, we obtain the following corollary.

\begin{corollary}[PAC-Learning under Isotropic Log-Concave Distributions]
  \label{cor:logconcave}
  Let $\D$ be a distribution on $\R^d \times \{\pm 1\}$ that satisfies the Tsybakov noise condition with parameters $(\tsyb,\tsya)$ and the marginal
  $D_{\bx}$ is an isotropic log-concave distribution. There exists an algorithm that  draws $N
  = d^{O(k)} \log\left(1/\delta\right)  $ examples from
  $\D$ where $k=O\left( \frac{1}{\tsyb^2} \log^2 \left(\tsya/\eps \right)  \right) $, runs in $\poly(N,d)$ time, and
  computes a vector $\wh{\vec w}$ such that
  $\err_{0-1}^{\D_{\bx}}(h_{\wh{\bw}}, f) \leq \eps$, with probability $1-\delta$,
  where $f$ is the target halfspace.
\end{corollary}

  We now provide a high-level sketch of the proof of Theorem~\ref{thm:main_angle} for constant values of the parameters 
  $L$, $R$, $\cona$, and $\conb$. 
  For every candidate halfspace $\vec w$, that has angle greater than $\eps$
  with the optimal hypothesis vector $\wstar$, our main structural result,
  Theorem~\ref{lem:polynomial_certificate}, guarantees that there exists a
  polynomial $p$ of degree $k = O((\log(1/\eps)/\alpha)^2)$
  such that
  $$
  \E_{(\bx, y) \sim \D}[p^2(\bx) \1_B(\bx) \dotp{\vec w}{\bx}y] \leq - \Omega(\eps)\;.
  $$
  Moreover, from Lemma~\ref{lem:empirical_objective_error}, we get that, given
  a candidate $\vec w$, we can compute a witnessing polynomial $p$ in time $d^{O(k)}$.
  The next step is to use the certificate to improve the candidate $\vec w$.
  We are going to use Online Projected Gradient Decent
  (OPGD) to do this.
  \begin{lemma}[see, e.g., Theorem 3.1 of \cite{hazan2016introduction}]
  \label{lem:online_optimization}
  Let ${\cal V}\subseteq \R^n$ a non-empty closed convex set with diameter $K$.
  Let $\ell_1,\ldots, \ell_T$ be a sequence of T convex functions $\ell_t: {\cal
  V}\mapsto \R$ differentiable in open sets containing $\cal V$, and let $G=\max_{t\in[T]}\snorm{2}{\nabla_{\bw} \ell_t}$.
  Pick any $\vec w_1\in \cal V$ and set $\eta_t=\frac{K}{G\sqrt{t}}$ for $t\in[T]$. Then, for all $\vec u\in \cal V$,
  we have that
  \begin{align*}
        \sum_{t=1}^{T}( \ell_t(\vec w_t) -\ell_t(\vec u))\leq \frac 32 GK\sqrt{T} \;.  
\end{align*}
  \end{lemma}
  In particular, let $p_t$ be the re-weighting function
  returned by Lemma~\ref{lem:empirical_objective_error}
  for a candidate $\sample{\bw}{t}$. If $\sample{\bw}{t}=\vec 0$, we set $p_t$ to be the zero function. 
  The objective function that we give to the online gradient descent algorithm, in the $t$-th step, is an estimator of
  $
  \ell_t(\sample{\vec w}{t}) =-
  \E_{(\bx, y)\sim \D}[(p_t(\bw)+\lambda) \dotp{\bw}{\bx}y]
  $, where $\lambda$ is a non-negative parameter. Using $\ell_t$, we perform a gradient update and project to get a new candidate $\sample{\bw}{t+1}$. The OPGD guarantees that after roughly $d^{\Theta(k)}$ steps, there exists a $t$,  where the value of function $\ell_t$ for our candidate is close to the value of the optimal one. From Theorem~\ref{lem:polynomial_certificate}, we know that this is possible only if the angle between the candidate and the optimal is less than $\eps$. For each iteration $t$, Step~\ref{alg:OPGDstep} of Algorithm~\ref{alg:OPGD} uses the OPGD algorithm, and the remaining steps are used to calculate the function $\ell_t$.

\begin{algorithm}[H]
	\caption{Learning Halfspaces with Tsybakov Noise}
	\label{alg:OPGD}
	\begin{algorithmic}[1]
		\Procedure{ALG}{$\eps,\delta$}
		\Comment{
			 $\eps$: accuracy, $\delta$:
          confidence}
          \State ${\vec w}^{(0)} \gets \vec e_1$
           \State $k\gets \Theta\left(
          	\frac{1}{\tsyb^2 R \conb} \log^2\left(  \frac{\cona \tsya}{\eps L R}\right)  \right) $
		\State $T \gets d^{\Theta(k)}  $
		\State \textbf{for} $t = 1, \dots, T$ \textbf{do}
\State\qquad $\eta_t \gets  \frac{1}{d^{ \Theta( k)} \sqrt{t} }$
          \State \qquad If $\vec w^{(t-1)}= \vec 0$ then
        \State \qquad \qquad $ p_t\gets 0$
          \State \qquad Else
        \State \qquad \qquad $p_t$ gets the output of SDP~\eqref{eq:sample_sdp} with input $\vec w^{(t-1)}/\snorm{2}{\vec w^{(t-1)}}$ \label{alg:poly}  \Comment{Lemma~\ref{lem:sample_sdp}}
        \State \qquad If SDP fails and $\vec w^{(t-1)}\neq \vec 0$ then
        \State \qquad \qquad \textbf{return} $ \vec w^{(t-1)}$
        \State\qquad  Draw $N=d^{\Theta(k)}\log\left(
       T/\delta\right) $ samples $\{(\sample{\bx}{1},\sample{y}{1}),\ldots,
        (\sample{\bx}{N},\sample{y}{N})\} $ from $\D$
        \State \qquad Set $\hat{\ell_t}(\vec w)$ according to Lemma~\ref{lem:algorithm_function_ell}
        \State\qquad  ${\vec w}^{(t)} \gets
        \Pi_{\cal V}\left({\vec w}^{(t-1)} - \eta_t \nabla_{\vec w}
          \hat{\ell_t}\left( {\vec w}^{(t-1)}\right)\right)$\label{alg:OPGDstep} \Comment{${\cal V}=\{\vec x \in \R^d : \snorm{2}{\vec x}\leq 1\}$} \EndProcedure
	\end{algorithmic}
\end{algorithm}
For the set $\mathcal{V}$, i.e., the unit ball with respect the $\snorm{2}{\cdot}$, the diameter $K$ equals to $2$. We are going
to show that in fact the optimal vector $\wstar$ and our current
candidate vector $\sample{\vec w}{t}$ have indeed a separation in the value of $\ell_t$. Because we do not have access to $\ell_t$ to optimize, we need a function $\hat{\ell}_t$, which is close to $\ell_t$ with high probability. The following lemma, which is proven in Appendix~\ref{ap:functionell}, gives us an efficient way to compute an approximation $\hat{\ell}_t$ of $\ell_t$.

\begin{lemma}[Estimating the function $\ell_t$]\label{lem:algorithm_function_ell}
	Let $p_t(\bx)$ be the non-negative function, given from the SDP~\eqref{eq:sample_sdp}. Then taking $d^{O(k)}
\log(1/\delta)$ samples, where $k=O\left(
\frac{1}{\tsyb^2 R \conb}  \log^2\left(\frac{\cona \tsya}{\eps L R}\right)  \right) $, we can efficiently compute a function $\hat{\ell_t}(\vec w)$ such that with probability at least $1-\delta$, the following conditions hold
\begin{itemize}
	\item $|\hat{\ell_t}(\bw)-\E_{(\bx,y) \sim \D} [(p_t(\vec
	x) +\lambda) y\dotp{\bw}{\vec x}]| \leq \eps$, for any $\lambda>0$ and $\bw \in \cal V$,
	\item $\snorm{2}{\nabla_{\bw} \hat{\ell_t}} \leq d^{O(k)}\;.$
\end{itemize}
\end{lemma}

The last thing we need to proceed to our main proof is to show that when 
the Algorithm~\ref{alg:OPGD} in Step~\ref{alg:poly} returns a function $p_t$, 
then there exists a function $\ell_t$ for which our current candidate vector $\sample{\bw}{t}$ 
and the optimal one $\wstar$ are not close.

\begin{lemma}[Error of $\ell_t$]\label{lem:expectation_error} Let $\sample{\vec
  w}{t}$ be a vector in $\mathcal{V}$ and $\wstar$ be the optimal vector. Let
  $g_t(\vec x)= -(p_t(\vec x)+ \lambda ) $ and $\ell_t(\vec w)= \E_{(\bx,y) \sim
  \D} [\dotp{g_t(\vec x) y {\vec x}}{\vec w}]$, where $p_t(\vec x)$ is a
  non-negative function such that $\E_{(\bx,y) \sim \D} [p_t(\vec
  x)y\dotp{\sample{\vec w}{t}}{\vec x}]\le -\snorm{2}{\sample{\vec
  w}{t}}\frac{\theta R}{16}$ and $\lambda$ a non-negative parameter. Then it holds
  $$\ell_t\left( \wstar \right) \le -\lambda \frac{R}{2}
 C_{\tsyb}^{\tsya}\left(\frac{R\;L}{2}\right)^{1/\tsyb} \quad\mathrm{ and } \quad  \ell_t({\vec
 w}^{(t)}) \geq \snorm{2}{{\vec
 w}^{(t)}}\left(\frac{R\theta}{16}-\lambda\right)\;.$$
\end{lemma}
\begin{proof}
   Without loss of generality, let $\wstar={\vec e_1}$. From Fact~\ref{obs:optimal_condition} and the definition of
  $\eta(\bx)$, for every $t\in[T]$, it holds $\ell_t(\wstar) \leq
  -\lambda\E_{\vec x\sim \D_{\bx}}[|\dotp{\wstar}{\vec x}|(1-2\eta(\bx))]$. To
  bound from above the expectation, we use the $\bounded$-bound properties. We have
  \begin{align}
    \E_{\vec x\sim \D_{\bx}}[|\dotp{\wstar}{\vec
    x}|(1-2\eta(\bx))]&\geq \frac{R}{2}\int_{R/2}^R (1-2\eta(\bx_1))
    \gamma(\bx_1) \d \bx_1 \geq \frac{R}{2}
    C_{\tsyb}^{\tsya}\left(\frac{R\;L}{2}\right)^{1/\tsyb}
    \nonumber\;,
  \end{align}
  where in the last inequality we used Lemma~\ref{lem:tsybakov_expectation}. Thus,
 $\ell_t\left( \wstar \right) \le -\lambda \frac{R}{2}
 C_{\tsyb}^{\tsya}\left(\frac{R\;L}{2}\right)^{1/\tsyb}$. From
 Lemma~\ref{lem:polynomial_certificate}, we have that
 \begin{align*}
   \ell_t({\vec w^{(t)}}) &=- \E_{(\bx,y) \sim \D} \left[(p_t\left( \bx
   \right)+\lambda) \dotp{{\vec w^{\left( t \right) }}}{\bx}y\right] \ge
   \snorm{2}{{\sample{\vec w}{t} }}\frac{R\theta}{16}- \E_{\bx \sim \D_{\x}}
   \left[\lambda \dotp{{\vec w^{(t)}}}{\bx}y\right] \\& \ge \snorm{2}{{\vec
     w^{(t)}}}\frac{R\theta}{16}- \lambda \sqrt{ \E_{\bx \sim \D_{\x}}
       \left[\dotp{\vec w^{(t)}}{\bx}^2\right] }   \geq \snorm{2}{{\vec
       w}^{(t)}}\left( \frac{R\theta}{16} -\lambda \right)  \;,
 \end{align*}
 where we used the Cauchy-Schwarz inequality and the fact that $\x$ is in isotropic position.
\end{proof}
We are now ready to prove our main results.
\begin{proof}[Proof of Theorem~\ref{thm:main_angle}] 
We start by setting all the parameters that we use in the proof. 
Let $k=\Theta\left(
	\frac{1}{\tsyb^2 R \conb} \log^2\left(  \frac{\cona \tsya}{\eps L R}\right)  \right)$ and 
	$\eps'=\eps \frac{R^2}{256}C_{\tsyb}^\tsya \left(\frac{R\
		L}{2}\right)^{\frac{1}{\tsyb}}$.
  Assume, in order to reach a contradiction, that for all steps $t$, $\theta\left( \sample{\vec w}{t},\wstar
  \right)\ge \eps $. Let $p_t(\bx)$ be the non-negative function output by the algorithm in Step
  \ref{alg:poly}. Then, from Lemma~\ref{lem:sample_sdp}, we have that 
  $\E_{(\bx,y) \sim \D} [p_t(\bx)y\dotp{\vec w^{(t)}}{\bx}]\leq
  -\snorm{2}{\vec w^{(t)}} \eps \frac{R}{16}$. Let $ \hat{\ell_t}(\vec w)$ be as in Lemma~\ref{lem:algorithm_function_ell}. 
  Then $\ell_t\left( \vec w
\right)=\E[\hat{\ell_t}(\vec w)]=  -\E_{(\bx,y) \sim
	\D} [\dotp{\left(p_t(\bx)
	+\lambda  \right) y {\vec x}}{\vec w}]$. Now using Lemma~\ref{lem:algorithm_function_ell},
for $N=\frac{d^{O(k)}}{\eps'^2}\log\left( \frac{T}{\delta}\right)$
  samples, we have  $\pr\left[ |\hat{\ell_t}(\vec w^{(t)})-\ell_t(\vec w^{(t)})|\geq \eps'\right]\leq
\frac{\delta}{2T}$ and $\pr\left[ | \hat{\ell_t}(\vec
w^{*})-\ell_t(\vec w^{*})|\geq \eps'\right]\leq
\frac{\delta}{2T}$.
From Lemma~\ref{lem:expectation_error}, for $\lambda
=\eps\frac{R}{32}$, in
each step $t$ we have $ \ell_t({\vec w}^{(t)}) \geq \snorm{2}{{\vec
w}^{(t)}}\frac{R}{32}\eps$ and $\ell_t\left( \wstar \right) \le -4\eps'$.
 From Lemma~\ref{lem:online_optimization}, for $G=d^{O(k)}$ and $K=2$, we get \[
   \sum_{t=1}^{T}\frac{\hat{\ell_t}\left( {\vec w^{(t)}} \right)  }{T} -
   \sum_{t=1}^{T}\frac{\hat{\ell_t}\left( {\vec w^*} \right) }{T} \leq
 \frac{3d^{O(k)}}{\sqrt{T}}\;.\] By the union bound, it follows that with probability at least $1-\delta$, we have that\[
     \sum_{t=1}^{T}\frac{\ell_t\left( {\vec w^{(t)}} \right)  }{T} -
   \sum_{t=1}^{T}\frac{\ell_t\left( {\vec w^*} \right) }{T} \leq
 \frac{3d^{O(k)}}{\sqrt{T}} +2\eps'\;.\]
 Thus, if the number of steps is $T= d^{\Theta(k)}/ \eps'^2$ then, with probability at least $1-\delta$ we have that,
 $\frac{1}{T}\sum_{t=1}^T\ell_t\left( \vec w^{\left( t \right) } \right)-
 \ell_t\left( {\vec w^*} \right)\leq 3\eps' $. This means that there
 exists $t\in [T]$ such that $\ell_t\left( {\vec w^{(t)}} \right)- \ell_t\left(
 {\vec w^*} \right)\leq 3\eps'$, which implies that
 $\ell_t\left( {\vec w^{(t)}} \right)  < -\eps'$ because from Lemma~\ref{lem:expectation_error} it holds
 $\ell_t\left( \wstar \right) \le -4\eps'$. Using the contrapositive of
 Theorem~\ref{lem:polynomial_certificate}, it follows that Step~\ref{alg:poly} does not return a witnessing
 function and also the $\vec w^{\left( t \right) }$ is not zero because then
 $\ell_t(\vec w^{(t)})=0$, which lead us to a contradiction. Therefore, we have that for the
 last $t$ it holds $\theta\left({\vec w^{(t)}},\wstar\right) \leq \eps$. Moreover,
 the number of samples is $O(T N)= (dk)^{O(k)} \log(1/\delta)$, and since $k$ is smaller than the dimension 
 we use $d^{O(k)} \log(1/\delta)$ samples.
\end{proof}
To prove the Theorem~\ref{thm:main_pac}, we need the following claim for the $\boundedU$-bounded distributions.
\begin{claim}[Claim 2.1 of \cite{DKTZ20}]
	\label{lem:angle_zero_one}Let $\D_{\bx}$ be an $\boundedU$-bounded distribution on $\R^d$. 
	Then, for any $0< \eps \leq 1$,\ we have that 
	$\err_{0-1}^{\D_{\bx}}(h_{\vec u},h_{\vec v})
	\leq U  \frac{\log^2\left( \frac{\cona}{\eps} \right)}{ \conb^2} \cdot \theta(\vec v, \vec u) + \eps \;.$
\end{claim}
\begin{proof}[Proof of Theorem~\ref{thm:main_pac}]
We run Algorithm~\ref{alg:OPGD} for $\eps' = \frac{\eps \conb^2 }{2 U} \frac{1}{
\log(2 / \eps)}$. From Theorem~\ref{thm:main_angle}, Algorithm~\ref{alg:OPGD} outputs a
$\hat{\vec w}$ such that $\theta(\hat{\vec  w},\wstar)\leq \frac{\eps \conb^2}{2 U}
\frac{1}{2 \log(1/ \eps)}$. From Claim~\ref{lem:angle_zero_one}, we have
that $\err_{0-1}( h_{\hat{\vec w}},f)\leq \eps$. This completes the proof.
\end{proof}

\bibliographystyle{alpha}
\bibliography{allrefs}
\clearpage
\appendix
\section{Omitted Proofs}\label{appendix:A}

\subsection{Proof of Lemma~\ref{lem:tsybakov_expectation}}\label{ap:claim33}

\begin{customlem}{\ref{lem:tsybakov_expectation}}
	\textit{
Let $\D$ be a distribution on $\R^{d} \times \{\pm 1\}$ that satisfies the Tsybakov noise condition with parameters $(\tsyb,\tsya)$.
Then for every measurable set $S \subseteq \R^d$ it holds
$
\E_{\vec x \sim D_\bx}[ \1_S(\bx) (1- 2 \eta(\bx))]
\geq
C_{\tsyb}^\tsya
\lp( \E_{\vec x \sim \D_\bx}[ \1_S(\bx)] \rp)^{\frac 1 \tsyb}
$, where
$ C_{\tsyb}^\tsya = \tsyb \left( \frac{1-\tsyb}{\tsya} \right)^{\frac{1-\tsyb}{\tsyb}}$.
	}
\end{customlem}

\begin{proof}
	We have
	\begin{align*}
	\E_{\vec x \sim D_\bx}[ \1_S(\bx) (1- 2 \eta(\bx))]
	&\geq t \E_{\vec x \sim D_\bx}[ \1_S(\bx) \1\{ 1 - 2 \eta(\bx) \geq t \} ] \\
	&\geq t \E_{\vec x \sim D_\bx}[ \1_S(\bx) ]
	- t \E_{\vec x \sim D_\bx}[ \1_S(\bx)  \1\{ 1 - 2 \eta(\bx) \leq t \} ]
	\\
	&\geq t \E_{\vec x \sim D_\bx}[ \1_S(\bx) ]
	- \tsya~ t^{\frac{1}{1- \tsyb}}\;.
	\end{align*}
	Let $A = \E_{\vec x \sim \D_\bx}[ \1_S(\bx)] $
	and set $t = \left(\frac{(1-\tsyb) A}{\tsya} \right)^{\frac{1-\tsyb}{\tsyb}}$.
	Then we have
	$$
	\E_{\vec x \sim D_\bx}[ \1_S(\bx) (1- 2 \eta(\bx))]
	\geq
	A^{1/\tsyb} \tsyb
	\left( \frac{1-\tsyb}{\tsya} \right)^{\frac{1-\tsyb}{\tsyb}}\;.
	$$
\end{proof}
\subsection{Proof of Fact~\ref{fct:chebyshev} and Lemma~\ref{lem:polynomial_norms} }\label{ap:factclaim}
\begin{customfc}{\ref{fct:chebyshev}}
	\textit{
 We denote by $T_k(t)$ the  degree-$k$ Chebyshev polynomial of the first kind.  It holds
\begin{align*}
T_k(t)
=
\begin{cases}
\cos(k\arccos t)\;, & |t| \le 1 \\
\frac12 \bigg( \Big(t-\sqrt{t^2-1} \Big)^k + \Big(t+\sqrt{t^2-1} \Big)^k \bigg)\;, \qquad & |t| \ge 1\;. \\
\end{cases}
\end{align*}
Moreover, it holds
$\snorm{2}{T_k}^2 \leq 2^{6k+2\log k +4}$.	
}
\end{customfc}
\begin{proof}
    Using that $\snorm{2}{T_k}^2\leq \snorm{1}{T_k}^2 $, we are going to show
    that $\snorm{1}{T_k}^2 \leq 2^{6k+2\log k + 4}$. We have that  
    \[
      \snorm{1}{T_k(t)}=\frac{k}{2}\sum_{i=1}^{\left\lfloor
      \frac{k}{2}\right\rfloor} 2^{k-2i} \binom{k-i}{i} \frac{1}{k-i} x^i
      \leq Fib(k+1)2^k \frac{k}{2} \leq \left( 1+\sqrt{5}  \right)^{k+1}2^k k 
    \;,\]
    where we used that $\sum_{i=1}^{\left\lfloor
    \frac{k}{2}\right\rfloor} \binom{k-i}{i} = Fib(k+1)$. Thus,
    $\snorm{1}{T_k}^2\leq 2^{6k+2\log k + 4}$.
\end{proof}

\begin{customlem}{\ref{lem:polynomial_norms}}
	\textit{  Let $p(t) = \sum_{i=0}^k c_i t^i$ be a degree-$k$ univariate polynomial.
		Given $\vec w \in \R^d$ with $\snorm{2}{\vec w} \leq 1$, define the multivariate polynomial
		$q(\vec x) = p(\dotp{\vec w}{\vec x}) = \sum_{S: |S| \leq k} C_S \bx^S$. It holds,
		$
		\sum_{S:|S| \leq k} C_S^2 \leq d^{2k} \sum_{i=0}^k c_i^2\, .
		$
		Moreover, let $r(t) = p(a t + b) = \sum_{i=0}^k d_i t^i$ for some $a, b \in \R$.  Then
		$
		\snorm{2}{r}^2
		\leq
		(2 \max(1,a) \max(1,b))^{2k} \snorm{2}{p}^2\, .
		$}
\end{customlem}
\begin{proof}
	We write
	$$
	q(\bx) =
	\sum_{i=0}^k c_i \dotp{\bw}{\bx}^i
	=
	\sum_{i=0}^k c_i
	\sum_{S:|S| = i} \frac{i!}{S!} \prod_{i=1}^d (x_i w_i)^{S_i}
	=
	\sum_{i=0}^k c_i
	\sum_{S:|S| = i} \frac{i!}{S!} \bw^S
	\bx^S \,.
	$$
	We have
	$$
	\sum_{i=0}^k \sum_{S:|S| = i} c_i^2  \left(\frac{i!}{S!}\right)^2 \bw^{2 S}
	\leq
	\sum_{i=0}^k c_i^2 \left(\sum_{S:|S| = i}  \frac{i!}{S!}\right)^2
	\leq d^{2k} \sum_{i=0}^k c_i^2\, ,
	$$
	where we used the fact that $|\bw_i| \leq 1$ for all $i$.
	To prove the second claim, we work similarly.  We have
	$$
	r(x) = \sum_{i=0}^k c_i \sum_{j=0}^i \binom{i}{j} a^j b^{i-j} x^j
	= \sum_{i=0}^k c_i \sum_{j=0}^i \binom{i}{j} a^j b^{i-j} x^j.
	$$
	We have
	$$
	\sum_{i=0}^k c_i^2
	\sum_{j=0}^i \left(\binom{i}{j} a^j b^{i-j}\right)^2
	\leq (2 \max(1, a) \max(1, b) )^{2k} \sum_{i=0}^k c_i^2\,.
	$$
\end{proof}

\subsection{Proof of Lemma~\ref{lem:algorithm_function_ell}} \label{ap:functionell}
\begin{customlem}{\ref{lem:algorithm_function_ell}}
\textit{	Let $p_t(\bx)$ be the non-negative function, given from the SDP~\eqref{eq:sample_sdp}. Then taking $d^{O(k)}
	\log(1/\delta)$ samples, where $k=O\left(
	\frac{1}{\tsyb^2 R \conb}  \log^2\left(\frac{\cona \tsya}{\eps L R}\right)  \right) $, we can efficiently compute a function $\hat{\ell_t}(\vec w)$ such that with probability at least $1-\delta$, the following conditions hold
	\begin{itemize}
		\item $|\hat{\ell_t}(\bw)-\E_{(\bx,y) \sim \D} [(p_t(\vec
		x) +\lambda) y\dotp{\bw}{\vec x}]| \leq \eps$, for any $\lambda>0$ and $\bw \in \cal V$,
		\item $\snorm{2}{\nabla_{\bw} \hat{\ell_t}} \leq d^{O(k)}\;.$
\end{itemize}}
\end{customlem}
\begin{proof}
For convenience, let $g_t(\x)= p_t(\x) + \lambda$. The proof is similar to Lemma~\ref{lem:empirical_objective_error}. Let $\hat{\ell_t}(\vec w)= \frac{1}{N}\sum_{i=1}^{N}\dotp{ g_t(\vec
	x^{(i)}) y^{(i)} \vec x^{(i)}}{\vec w}$ and $\ell_t(\bw)=\E_{(\bx,y) \sim \D} [\dotp{g_t(\x) y \x }{\bw}]$.
Then from Cauchy-Schwarz we have
$$|\hat{\ell_t}(\bw)-\ell_t(\bw) |\leq \snorm{2}{\frac{1}{N}\sum_{i=1}^{N} g_t(\vec
		x^{(i)}) y^{(i)} \vec x^{(i)}-\E_{(\bx,y) \sim \D} [g_t(\x) y \x ]} \snorm{2}{\bw}\;.$$
We have that $\snorm{2}{\bw}\leq 1$, thus we need to prove that 
\begin{equation}\label{eq:apendix_helper}
\pr\left[\snorm{2}{\frac{1}{N}\sum_{i=1}^{N} g_t(\vec
	x^{(i)}) y^{(i)} \vec x^{(i)}-\E_{(\bx,y) \sim \D} [g_t(\x) y \x ]}>\eps \right]\leq \delta\;.
\end{equation} 
Let $\vec M_j=\E_{(\bx, y) \sim \D}[m(\x)m(\x)^T \1_{B}(\x)]\x_j$ and  $\widetilde{\vec M_j}
= \frac{1}{N} \sum_{i=1}^N
\vec m(\sample{\bx}{i}) \vec m(\sample{\bx}{i})^T \1_B(\sample{\bx}{i})
\sample{\bx_j}{i}$, and then $\vec A$ be a matrix such that $ \tr\left( \vec A \vec M_j  \right)= \E_{(\bx,y) \sim \D} [p_t(\x) y \x_j ]$, i.e., the matrix of the coefficients of the polynomial and assume that $\snorm{F}{\vec A} \leq Q$, where $Q=d^{O(k)}$. Using the same proof ideas as in Lemma~\ref{lem:empirical_objective_error}, we get
  $$
\tr\left( \vec A (\vec M_j - \widetilde{\vec M_j}) \right)
\leq
\snorm{F}{\vec A} \snorm{F}{\vec M_j - \widetilde{\vec M_j}}
\;. $$
 Therefore, it suffices to bound the probability
that $\snorm{F}{\vec M_j - \widetilde{\vec M_j}} \geq \eps/( 2 dQ)$.
From Markov's inequality, we have
\begin{align*}
\pr\left[\snorm{F}{\vec M_j - \widetilde{\vec M_j}}
\geq \eps/(2 dQ) \right]
\leq \frac{4 d^2 Q^2}{\eps^2} \E\left[\snorm{F}{\vec M_j - \widetilde{\vec M_j}}^2\right]\,.
\end{align*}
Using Equation~\eqref{eq:momemt_matrix_frobenius_bound} (which holds in our case as well and is proved the same way by setting $\vec w=\vec e_j$), we get
\begin{align*}
\pr\left[\snorm{F}{\vec M_j - \widetilde{\vec M_j}}
\geq \eps/( 2dQ) \right]
\leq \frac{4 d^2 Q^2}{\eps^2}  \frac{1}{N}
\cona {(\conb/2)}^{-2 k} (d+k)^{3k+1}\,.
\end{align*}
Then, for $N \geq \cona d^3 Q^2
{(\conb/2)}^{-2 k} (d+k)^{3k+1}/(4\eps^2)$ samples we can estimate $\vec M_j$
within the target accuracy with probability at least $1-1/(8d)$.
Now we are going to give a loose bound for the 
\begin{equation*}
\pr\left[\snorm{2}{\frac{1}{N}\sum_{i=1}^{N}\lambda y^{(i)} \vec x^{(i)}-\E_{(\bx,y) \sim \D} [\lambda y \x ]}>\eps \right]\leq \delta\;.
\end{equation*} 
Using the same argument as before, we have from Markov's inequality, that
\begin{align*}
\pr\left[\snorm{2}{\frac{1}{N}\sum_{i=1}^{N} y^{(i)} \vec x^{(i)}-\E_{(\bx,y) \sim \D} [y \x ]
}\geq \eps/(2d \lambda) \right]
\leq \frac{4 d^2 \lambda^2}{\eps^2} \E\left[\snorm{2}{\frac{1}{N}\sum_{i=1}^{N} y^{(i)} \vec x^{(i)}-\E_{(\bx,y) \sim \D} [y \x ]
}^2 \right]\,.
\end{align*}
Using the linearity of expectation, we have
\begin{align*}
	\E\left[\snorm{2}{\frac{1}{N}\sum_{i=1}^{N} y^{(i)} \vec x^{(i)}-\E_{(\bx,y) \sim \D} [y \x ]
	}^2 \right]\leq \sum_{j=1}^d \E\left[\left(\frac{1}{N}\sum_{i=1}^{N} y^{(i)} \vec x_j^{(i)}-\E_{(\bx,y) \sim \D} [y \x_j]
\right)^2 \right]\leq \sum_{j=1}^d \var\left[\frac{1}{N}\sum_{i=1}^{N} y^{(i)} \vec x_j^{(i)}\right]  \;.
\end{align*}
Then, using the fact that $\x$ is in isotropic position, we have
\begin{align*}
	\var\left[\frac{1}{N}\sum_{i=1}^{N} y^{(i)} \vec x_i^{(i)}\right]  \leq \frac{1}{N}\E_{(\bx, y) \sim \D}[(\vec x_i^{(i)}y)^2 ]=1/N\;.
\end{align*}
Thus, for $N>4d^3\lambda^2/\eps^2$, with probability at least $1-1/8$, we have that 
$$\snorm{2}{\frac{1}{N}\sum_{i=1}^{N}\lambda y^{(i)} \vec x^{(i)}-\E_{(\bx,y) \sim \D} [\lambda y \x ]}\leq\eps/2 \;.$$ 
Putting everything together and by the union bound, we have that for 
$N> \max( \cona d^3 Q^2 {(\conb/2)}^{-2 k} (d+k)^{3k+1}/(4\eps^2),4d^3\lambda^2/\eps^2)$, with probability $3/4$, we have that
\begin{align*}
	\snorm{2}{\frac{1}{N}\sum_{i=1}^{N} g_t(\vec
		x^{(i)}) y^{(i)} \vec x^{(i)}-\E_{(\bx,y) \sim \D} [p_t(\x) y \x ]}&\leq \snorm{2}{\frac{1}{N}\sum_{i=1}^{N} p_t(\vec
		x^{(i)}) y^{(i)} \vec x^{(i)}-\E_{(\bx,y) \sim \D} [g_t(\x) y \x ]}
	\\&+\snorm{2}{\frac{1}{N}\sum_{i=1}^{N} \lambda y^{(i)} \vec x^{(i)}-\E_{(\bx,y) \sim \D} [\lambda y \x ]}\leq \eps/2 + \eps/2=\eps\;.
\end{align*}

To amplify the confidence probability to $1-\delta$, we can use the above empirical estimate $\ell$
times to obtain estimates 
$\sample{\wt {\vec M_j}}{1}, \ldots, \sample{\wt{\vec M_j}}{\ell}$ for all $j\in[d]$ and keep the median as our final estimate. 
It follows that $\ell =O(\log(d/\delta))$ repetitions suffice to guarantee
confidence probability at least $1-\delta$.

To prove the second statement, from Equation~\eqref{eq:apendix_helper}, we have that with probability $1-\delta$
$$ \snorm{2}{\nabla_{\bw} \hat{\ell_t}} \leq  \snorm{2}{\nabla_{\bw} {\ell_t}} + \eps \leq d^{O(k)} +\eps=d^{O(k)}\;, $$
where we used Theorem~\ref{lem:polynomial_certificate}.
This completes the proof.
\end{proof}

\appendix

\end{document}